%% file: linear_asp.tex
\documentclass{new_tlp}
\usepackage{mathptmx}

\usepackage[german,english]{babel}

\usepackage{url}
\usepackage{color}

\usepackage{amsmath}
\usepackage{amsfonts}
\usepackage{graphicx}   
\usepackage{subfigure}  
\usepackage{epstopdf}  

\input{macro_random}

\title{Random Logic Programs: Linear Model}

\author[K. Wang, L. Wen and K. Mu]
{Kewen Wang$^1$, Lian Wen$^1$ and Kedian Mu$^2$ \\
$^1$School of Information and Communication Technology\\
Griffith University, Australia \\
\email{\{k.wang,l.wen\}@griffith.edu.au}\\
$^2$School of Mathematical Sciences\\
Peking University, China \\
\email{mukedian@math.pku.edu.cn}}

\submitted{17 June 2013}
 \revised{21 March 2014, 10 June 2014}
 \accepted{12 June 2014}
 
\begin{document}
\label{firstpage}
\maketitle

\begin{abstract}
This paper proposes a model, the linear model, for randomly generating logic programs with low density of rules and investigates statistical properties of such random logic programs. It is mathematically shown that the average number of answer sets for a random program converges to a constant when the number of atoms approaches infinity. Several experimental results are also reported, which justify the suitability of the linear model. It is also experimentally shown that, under this model, the size distribution of answer sets for random programs tends to a normal distribution when the number of atoms is sufficiently large.

\end{abstract}

\input{introduction}
\input{preliminaries}
%
%
\input{theorems}
\input{proof}
%
\input{experiment}

%
\input{conclusion}
\input{appendix}


\input{refs_linear}
\label{lastpage}
\end{document}

%% file: macro_random.tex
\newcommand{\wrt}{w.~r.~t.}

\newcommand{\linear}[2]{\ensuremath{\mathsf{L}(N2, #1, #2)}}

\newcommand{\simple}[1]{\ensuremath{\mathsf{simple}(#1)}}

\newcommand{\model}[1]{\ensuremath{\mathsf{Mod}(#1)}}

\newcommand{\AS}[1]{\ensuremath{\mathsf{AS}(#1)}} 

\newcommand{\la}{\leftarrow}

\newcommand{\D}{\;|\;}

\newtheorem{definition}{Definition}
 \newtheorem{theorem}{Theorem}
 \newtheorem{lemma}{Lemma}
\newtheorem{proposition}{Proposition}

\newcommand\nop[1]{}

\newcommand{\head}[1]{\ensuremath{\mathit{head}(#1)}}
\newcommand{\pbody}[1]{\ensuremath{\mathit{body}^{+}(#1)}}
\newcommand{\nbody}[1]{\ensuremath{\mathit{body}^{-}(#1)}}

\newcommand{\body}[1]{\ensuremath{\mathit{body}(#1)}}

\newcommand{\nafo}[0]{\mathit{not}}
\newcommand{\naf}[0]{\nafo\;}



%% file: introduction.tex
\section{Introduction}\label{sec:introduction}

As in the case of combinatorial structures, the study of randomly generated instances of NP-complete problems
in artificial intelligence has received significant attention in the last two decades.
These problems include the satisfiability of boolean formulas
(SAT) and the constraint satisfaction problems (CSP) \cite{AchlioptasKKKMS97,AchlioptasNP05,CheesemanKT91,GentW94,HubermanH87,MitchellSL92,MonassonZKST99}.
In turn, these results on properties of random SAT and random CSP significantly help
researchers in better understanding SAT and CSP, and developing fast solvers for them.

On the other hand, it is well known that reasoning in propositional logic and in most
constraint languages is \emph{monotonic} in the sense that
conclusions obtained before new information is added cannot be
withdrawn. However, commonsense knowledge is \emph{nonmonotonic}. In
artificial intelligence, significant effort has been paid to develop
fundamental problem solving paradigms that allow users to
conveniently represent and reason about commonsense knowledge and
solve problems in a declarative way. \emph{Answer set programming
(ASP)} is currently one of the most widely used nonmonotonic
reasoning systems due to its simple syntax, precise semantics and
importantly, the availability of ASP solvers, such as clasp \cite{GebserKS09}, dlv \cite{LeonePFEGPS06},
and smodels \cite{SyrjanenN01}. However, the theoretical study of random ASP has not made much progress so far \cite{NamasivayamT09,Namasivayam09,SchlipfTW2005,ZhaoL03}. 

\cite{ZhaoL03} first conducted an experimental study on the issue of phase
transition for randomly generated ASP programs whose rules can have
three or more literals.
\cite{SchlipfTW2005} reported on their experimental work for determining the distribution of
randomly generated normal logic programs at the Dagstuhl Seminar. 

To study statistical properties for random programs, \cite{NamasivayamT09,Namasivayam09} considered the class of randomly
generated ASP programs in which each rule has exactly two literals, called simple random programs.
Their method is to map some statistical properties of random graphs into simple random programs
by transforming a random program into that of a random graph through a close connection between simple random programs and random graphs. As the authors have commented, those classes of random programs that correspond to some classes of random graphs are too restricted to be useful. 
Their effort further confirms that it is challenging to recast statistical properties of SAT/CSP to nonmonotonic formalisms such as ASP.  

In fact, 
the monotonicity plays an important role in proofs of major results for random SAT/CSP.
Specifically, major statistical properties for SAT/CSP are based on a simple but important property: 
An interpretation $M$ is a model of a set of clauses/constraints if and only if $M$ is a model of each clause/constraint.
Due to the lack of monotonicity in ASP, this property fails to hold for ASP and other major nonmonotonic formalisms.

For this reason, it might make sense to first focus on some relatively simple but expressive classes of ASP programs (i.e., still NP-complete).
We argue that the class of \emph{negative two-literal programs} (i.e. normal logic programs in which a rule body has exactly one negative literal)
is a good start for studying random logic programs under answer set semantics for several reasons\footnote{Our definition of \emph{negative two-literal programs} here is slightly different from that used by some other authors. But these definitions are essentially equivalent if we notice that a fact rule $a\la$ can be expressed as a rule $a\la\naf a'$ where $a'$ is a new atom. Details can be found in Section~\ref{sec:preliminaries}.}:
(1) The problem of deciding if a negative two-literal program has an answer set is still NP-complete. 
In fact, the class of negative two-literal programs is used to show the NP-hardness of answer set semantics for normal logic programs in \cite{MarekT91} (Theorem 6.4 and its proof, where a negative two-literal program corresponds to a \emph{simple $K_1$-theory}).
(2) Many important NP-complete problems can be easily encoded as (negative) two-literal 
programs \cite{HuangJLY02}.
(3) Negative two-literal programs allow us to conduct large scale experiments with existing ASP solvers, such as smodels, dlv and clasp.


In this paper we introduce a new model for generating and studying random negative two-literal programs, called
\emph{linear model}. A random program generated under the linear model is of the size about $c\times n$ where $c$ is a constant and $n$ is the total number of atoms. We choose such a model of randomly generating negative two-literal programs for two reasons. First, if we use a natural way to randomly generate programs like what has been done in SAT and CSP, we would come up with two possible models in terms of program sizes (i.e. linear in $n$ and quadratic in $n$), since only $n^2$ negative two-literal rules in total can be generated from a set of $n$ atoms. We study statistical properties of such random programs and have obtained both theoretical and experimental results for random programs generated under the linear model, especially, Theorem~\ref{thm:mainthm}. These properties include the average number of answer sets, the size distribution of answer sets, and the distribution of consistent programs under the linear model.
Second, such results can be used in practical applications. For instance, it is important to compute all answer sets of a program in applications, such as diagnoses and query answering, in P-log \cite{P-Log}. In such cases, the number of answer sets for a program is certainly relevant. 
If we know the number of answer sets and the average size of the answer sets for a logic program, such information can be useful heuristics for finding all answer sets of a given program. 
Also, the linear model of random programs may be useful in application domains such as ontology engineering where most of large practical ontologies are sparse in the sense that the ratio of terminological axioms to concepts/roles is relatively small \cite{StaabS04}. 
  
The contributions of this work can be summarised as follows:
\begin{enumerate}
\item A model for generating random logic programs, called the linear model, is established. 
Our model generates random logic programs in a similar way as SAT and CSP,  but we distinguish the probabilities for picking up pure rules and contradiction rules. 
\cite{NamasivayamT09} discusses some program classes of two-literal programs that may not be negative. However, as their major results are inherited from the corresponding ones in random graph theory, such results hold only for very special classes of two-literal programs. For instance,  in regard to the result on negative two-literal programs without contradiction rules (Theorem 2, page 228), the authors pointed out that the theorem ``concerns only a narrow class of dense programs, its applicability being limited by the specific number of rules programs are to have'' ($0<c<1$, $x$ is a fixed number, the number of rules $m= \lfloor cN + x\sqrt{c(c-1)N}\rfloor$ and $N=n(n-1)$)\footnote{There may be an error here as $c-1 <0$.}. 
\item We mathematically show that the average number of answer sets for a random program converges to a constant when the number of atoms approaches infinity. 
We note that the proofs of statistical properties, such as phase transitions, for random SAT and random CSP are usually obtained through the independence of certain probabilistic events, which in turn is based on a form of the monotonicity of classical logics (specifically, given a set of formulas $S=\{\phi_1,\ldots,\phi_t\}$ with $t\ge 0$, it holds that $\model{S} = \model{\phi_1}\cap \cdots \cap \model{\phi_t}$ when $\model{\cdot}$ denotes the set of all models of a formula or a set of formulas). However, it is well known that ASP is nonmonotonic. In our view, this is why many proof techniques for random SAT cannot be immediately adapted to random ASP. In order to provide a formal proof for Theorem~\ref{thm:mainthm}, we resort to some techniques from mathematical analysis such as Stirling's Approximation and Taylor series. As a result, our proof is both mathematically involved and technically novel. We look into the application of our main result in predicting the consistency of random programs (Proposition~\ref{pro:consistency} and Section~\ref{subsec:3}). 

\item We have conducted significant experiments on statistical properties of random programs generated under the linear model. These properties include the average number of answer sets, the size distribution of answer sets,  and the distribution of consistent programs under the linear model. For the average number of answer sets, our experimental results closely match the theoretical results obtained in Section~\ref{sec:theorem}. 
Also, the experimental results corroborate the conjecture that under the linear model, the size distribution of answer sets for random programs obeys a normal distribution when $n$ is large. The experimental results show that our theories can be used to predict practical situations.
%
%
As explained above, we need to find all answer sets in some applications. 
For large logic programs, it may be infeasible to find all answer sets but we could develop algorithms for finding most of the answer sets. If we know an average size of answer sets, we might need only to examine those sets of atoms whose sizes are around the average size.

\end{enumerate}


The rest of the paper is arranged as follows. In Section~\ref{sec:preliminaries}, we briefly review answer set semantics of logic programs and some properties of two-literal programs that will be used in the subsequent sections. In Section~\ref{sec:theorem}, we first introduce the linear model for random logic programs (negative two-literal programs), study mathematical properties of random programs, and then present the main result in a theorem. 
In Section~\ref{sec:experiment} we describe some of our
experimental results and compare them with related theoretical results obtained in the paper. We conclude the work in Section~\ref{sec:conclusion}.
For the convenience of readers, some mathematical basics required for the proofs are included in the Appendix at the end of the paper. 

%% file: preliminaries.tex
\section{Answer Set Semantics and Two-Literal Programs} \label{sec:preliminaries}
We briefly review some basic definitions and
notation of answer set programming (ASP). We restrict our discussion
to finite propositional logic programs on a finite set $A_n$
of $n$ atoms ($n>0$).

A \emph{normal logic program} (simply, \emph{logic program}) is a
finite set of rules of the form
\begin{equation}\label{rule}
a \leftarrow b_1,\ldots,b_s,\naf c_1,\ldots,\naf c_t,
\end{equation}
where $\naf$ is for the default negation, $s, t\ge 0$, and $a$,
$b_i$ and $c_j$ are atoms in $A_n$ ($i=1, \ldots, s$, $j=1, \ldots,
t$).

We assume that all atoms appearing in the body of a rule are pairwise distinct.


A \emph{literal} is an atom $a$ or its default negation $\naf a$.
The latter is called a \emph{negative literal}. An atom
$a$ and its default negation $\naf a$ are said to be {\em complementary}.

Given a rule $R$ of form (\ref{rule}), its head is defined as
$\head{R}=a$ and its body is $\body{R}=\pbody{R}\cup \naf \nbody{R}$
where $\pbody{R}=\{b_1,\ldots,b_s\}$, $\nbody{R}=\{c_1,\ldots,c_t\}$, and $\naf \nbody{R}=\{\naf q \D q\in\nbody{R}\}$.

A rule $R$ of form (\ref{rule}) is \emph{positive}, if $t=0$;
\emph{negative}, if $s=0$. A logic program $P$ is called
\emph{positive} (resp.\ \emph{negative}), if every rule in $P$ is
positive (resp.\ negative).



An \emph{interpretation} for a logic program $P$ is a set of
atoms $S \subseteq A_n$. A rule $R$ is satisfied by $S$, denoted
$S\models R$, if $S\models \head{R}$ whenever $\pbody{R}\subseteq S$
and $\nbody{R}\cap S = \emptyset$. Furthermore, $S$ is a model
of $P$, denoted $S\models P$, if $S\models R$ for every rule $R \in
P$. A model $S$ of $P$ is a \emph{minimal model} of $P$ if for any
model $S'$ of $P$, $S'\subseteq S$ implies $S'=S$.

The semantics of a logic program $P$ is defined in terms of its
\emph{answer sets} (or equivalently, \emph{stable models}) \cite{gellif88,gellif90a} as
follows. Given an interpretation $S$, the {\em reduct} of $P$ on $S$
is defined as $P^S = \{\head{R}\la\pbody{R} \D R\in P, \nbody{R}\cap
S =\emptyset\}$. Note that $P^S$ is a positive logic program and
every (normal) positive program has a unique least model. Then we say $S$
is an {\em answer set} of $P$, if $S$ is the least model of $P^S$.
By $\AS{P}$ we denote the collection of all answer sets of $P$.
For an integer $k\ge 0$, $\AS{P,k}$ denotes the set of answer sets of size $k$ for $P$.

A logic program $P$ may have zero, one or multiple answer sets. $P$
is said to be {\em consistent}, if it has at least one answer set. It is
well-known that the answer sets of a logic program $P$ are
incomparable: for any $S$ and $S'$ in $\AS{P}$, $S \subseteq S'$
implies $S=S'$.

Two logic programs $P$ and $P'$ are \emph{equivalent} under answer set semantics, denoted
$P\equiv P'$, if $\AS{P} = \AS{P'}$, i.e., $P$ and $P'$ have the
same answer sets. We can slightly generalise the equivalence of two programs as follows.
Let $P$ be a logic program on $A_n$ and $P'$ a logic program on $A_n\cup E$, where $E$ is a set
of new (auxiliary) atoms. We say $P$ and $P'$ are equivalent if the following two conditions are satisfied:

\begin{enumerate}
\item if $S\in \AS{P}$, then there exists $S'\in\AS{P'}$ such that
$S' = S \cup S_e$ and $S_e \subseteq E$.
\item if $S'\in \AS{P'}$, then $S = S' \setminus E$ is in
$\AS{P}$.
\end{enumerate}

From the next section and on, we will focus on a special class of logic programs, 
called \emph{negative two-literal programs}. 

A \emph{negative two-literal rule} is a rule of the form $a\la \naf b$ where $a$ and $b$ are atoms. 
These two atoms do not have to be distinct. If $a\neq b$, it is a \emph{pure rule}; if $a=b$, it is a \emph{contradiction rule}. 
A \emph{negative two-literal program} is a finite set of negative two-literal rules.

We note that our definition is slightly different from some other authors, such as \cite{Janhunen06,LoncT04}, in that fact rules are not allowed in our definition. This may not be an issue since a fact rule of the form $a\la$ can be expressed as a negative two-literal rule $a\la \naf c$, where $c$ is a new atom that does not appear in the program.


%
%

It is shown in \cite{MarekT91} that the problem of deciding the existence of answer sets for a negative two-literal program is NP-complete. This result confirms that the class of negative two-literal programs is computationally powerful and it makes sense to study the randomness for such a class of logic programs.

We remark that, by allowing the contradiction rules, constraints of the form $\la b_1,\ldots,b_s,\naf c_1,\ldots,\naf c_t$ ($s, t\ge 0$) can be expressed in the class of negative two-literal programs.
A contradiction rule $a\la \naf a$ is strongly equivalent to the constraint
$\la \naf a$ under answer set semantics: for any logic program $P$, $P\cup\{\la \naf a\}$ is equivalent to $P\cup\{a\la \naf a\}$ under answer set semantics.
Notice also that a constraint of the form $\la \naf a, \naf b$ is strongly
equivalent to the two constraints $\la \naf a$ and $\la \naf b$, and
a constraint of the form $\la a$ is strongly equivalent to 
two rules $\la \naf a'$ and $a'\la \naf a$ where $a'$ is a fresh atom.

In the rest of this section we present three properties of negative two-literal programs.  
While Proposition~\ref{thm:simple:form} is to demonstrate the expressive power of negative two-literal programs,
Propositions~\ref{pro:structure} and \ref{pro:no:0:n} will be used to prove our main theorem in the next section. These properties are already known in the literature and we do not claim their originality here.

%

First, each logic program can be equivalently
transformed into a negative two-literal program under answer set semantics.
This result is mentioned in \cite{BlairDJRS99} but no proof is provided there. 
For completeness, we provide a proof of this proposition in the appendix at the end of the paper.

\newcommand{\translationTwo}{
Each normal logic program $P$ is equivalent to a negative two-literal program under answer set semantics.
}
%
\begin{proposition}\label{thm:simple:form}
\translationTwo
\end{proposition}
%

The next result provides an alternative characterization for the answer sets of a negative two-literal program, which is a special case of
Theorem 6.81, Section 6.8 in \cite{MarekT93}.

\newcommand{\aspChar}{
Let $P$ be a negative two-literal program on $A_n$ containing at least one rule.
Then $S$ is an answer set of $P$
iff the following two conditions are satisfied:
\begin{enumerate}
\item If $b_1, b_2\in A_n\setminus S$,  then $b_1 \la\naf b_2 $ is not a rule in $P$.
\item If $a \in S$, then there exists $b \in A_n\setminus S$ such that  $a\la\naf b$
is a rule in $P$.
\end{enumerate}
}
%
\begin{proposition}\label{pro:structure}
\aspChar
\end{proposition}
%
We note that in condition 1 above, it can be the case that $b_1=b_2$. 

We note that if the empty set is an answer set of a negative two-literal program, the program must be empty. Also, $A_n$ is not an answer set for any negative two-literal program on $A_n$.

\newcommand{\aspRange}{
Let $P$ be a negative two-literal program on $A_n$ containing at least one rule. If $S$ is an answer set of $P$, then $0<|S|<n$. Here $|S|$ is the number of elements in $S$.}
%
\begin{proposition}\label{pro:no:0:n}
\aspRange
\end{proposition}
%

%% file: theorems.tex
\section{Random Programs and Their Properties} \label{sec:theorem}

In this section we first introduce a model for randomly generating negative two-literal programs and then present some statistical properties of such random programs. The main result in this section (Theorem~\ref{thm:mainthm}) shows that the expected number of answer sets for a random program on $A_n$ generated under our model converges to a constant when the number $n$ of atoms approaches infinity.
As the proof of Theorem~\ref{thm:mainthm} is lengthy and mathematically involved, some technical details, as well as necessary basics of mathematical analysis,  are included in the appendix at the end of the paper. 

In this section, we assume that each negative two-literal program contains at least one rule.


%
\begin{definition}[Linear Model $\linear{c_1}{c_2}$]
\label{def:P}
Let $c_1$ and $c_2$ be two non-negative real numbers with $c_1 + c_2 >0$. Given a set $A_n$ of $n$ atoms with $n>\max(c_1,c_2)$, a \emph{random program} $P$ on $A_n$ is a negative two-literal program that is generated as follows:
\begin{enumerate}
\item For any two different atoms $a, b \in A_n$, the probability of the pure rule $a \la \naf b$ being in $P$ is $ p=c_1/n$. 
\item For any atom $a \in A_n$, the probability of the constraint $a\la \naf a$ being in $P$ is $d=c_2/n$. 
\item Each rule is selected randomly and independently based on the given probability.
\end{enumerate}
\end{definition}
%

In the above notation, `$N2$' is for `negative two-literal programs'. For simplicity, we assume that a random program is non-empty.
If $c_2=0$, then a random program generated under $\linear{c_1}{c_2}$ does not contain any contradiction rules.

In probability theory, the expected value (or mathematical expectation) of a random variable is the weighted average of all possible values that this random variable can take on.
Suppose random variable $X$ can take $k$ possible values $x_1, \ldots, x_m$ and each $x_k$ has the probability $p_k$ for $k=1,\ldots, m$.
Then the expected value of random variable $X$ is defined as
\[
E[X] = \sum\limits_{k=1}^m p_k x_k.
\]
Also, if a random variable $X$ is the sum of a finite number of other variables $X_1, \ldots, X_s$ ($s>0$), i.e.,
\[
X = \sum\limits_{k=1}^s X_k,
\]
then
\[
E[X] = \sum\limits_{k=1}^s E[X_k].
\]

The number $|P|$ of rules in random program $P$ (i.e., the size of $P$) is a random variable. As there are $n(n-1)$ possible pure rules, each of which has probability $p=c_1/n$, and $n$ possible constraints, each of which has the probability $d=c_2/n$.
Thus, the expected value of $|P|$, also called the \emph{expected number of rules} for random program $P$, is the sum of expected number of pure rules and the expected number of constraints:
\[
E[|P|]=n(n-1)p+nd=c_1(n-1)+c_2. 
\]
This means that the average size of random programs generated under the model $\linear{c_1}{c_2}$ is a linear function of $n$.
This is the reason why we refer to our model for random programs as the \emph{linear model} of random programs under answer sets.


For $S\subseteq A_n$ with $|S|=k$ ($0< k< n$), the probability of $S$ being an answer set of random program $P$, denoted $\Pr(k)$, can be easily figured out as the next result shows. We remark that, by Proposition~\ref{pro:no:0:n}, for negative two-literal program $P$, neither the empty set $\emptyset$ nor $A_n$ can be an answer set of $P$. So we do not need to consider the case of $k=0$ or $k=n$.
%
\begin{proposition}
\label{pro:prob:as}
Let $P$ be a random program on a set $A_n$ of $n$ atoms, generated under $\linear{c_1}{c_2}$, with $n>\max(c_1,c_2)$.
Then
\begin{equation} \label{eq:Pr:k1}
\Pr(k)=\left(1-\frac{c_1}{n}\right)^{(n-k)(n-k-1)}\left(1-\left(1-\frac{c_1}{n}\right)^{n-k}\right)^k
\left(1-\frac{c_2}{n}\right)^{n-k}.
\end{equation}
Recall that $p=c_1/n$ and $d=c_2/n$. If we denote $q=1-p$,
then Eq.(\ref{eq:Pr:k1}) can be simplified into 
\begin{equation} \label{eq:Pr:k}
\Pr(k)=q^{(n-k)(n-k-1)}(1-q^{n-k})^k(1-d)^{n-k}.
\end{equation}
\end{proposition}
%

\begin{proof}
Let $S$ be a subset of $A_n$ with $|S|=k$ and $T=A_n\setminus S$. 
We can split the first condition in Proposition~\ref{pro:structure} into two sub-conditions. $S$ is an answer set of negative two-literal program $P$ iff the following two conditions are satisfied:

\begin{enumerate}
\item [(1)] 
	\begin{enumerate}
	\item [(1.1)] for each pair $ b_1, b_2\in{T}$ with $b_1 \neq b_2 $, the rule $b_1 \la not ~b_2 $ is not in $P$.
	\item [(1.2)] for each $ a\in{T}$, the rule $a \la \naf a$ is not in $P$.
	\end{enumerate}
\item [(2)] for each $a \in{S}$, there exists an atom $b \in{T}$ such that  $a\la \naf b$ is in $P$.
\end{enumerate}

Let us figure out the probabilities that the above conditions (1.1), (1.2) and (2) hold, respectively.

We say that an atom $a$ is \emph{supported \wrt~ $S$ in $P$} (or just, \emph{supported}) if there exists a rule of the form $a\la \naf b$ in $P$ such that $b\in T$. In this case, the rule $a\la \naf b$ is referred to as a \emph{supporting rule} for $a$.

First, since $T$ contains $n-k$ elements, there are $(n-k)(n-k-1)$ possible pure rules of the form $b_1 \la not ~b_2 $ with $b_1, b_2\in{T}$ and $b_1 \neq b_2 $. By the definition of $\linear{c_1}{c_2}$, the probability that a pure rule does not belong to $P$ is $1-p=q$.
Thus, the probability that none of the pure rules with $b_1, b_2\in{T}$ and $b_1 \neq b_2 $ belongs to $P$ is $q^{(n-k)(n-k-1)}$.
That is, the condition (1.1) will hold with the probability $q^{(n-k)(n-k-1)}$.

Next, by the definition of $\linear{c_1}{c_2}$, the probability that a constraint rule of the form ${a\la\naf{a}}$ does not belong to $P$ is $1-d$.
Since $T$ contains $n-k$ atoms, the probability that none of the constraint rules of the form $a\la \naf a$ with $a\in T$ is $(1-d)^{n-k}$.
That is, the condition (1.2) will hold with the probability $(1-d)^{n-k}$.

Last, we consider the condition (2). For each $a\in S$, if a pure rule supports $a$, then it must be of the form $a \la \naf b$ for some $b\in T$. There are $n-k$ possible such pure rules. Also, $a$ is not supported by such pure rules only if $P$ does not contain such rules at all. Thus, the probability that $a$ is not supported (by one of such pure rules) is $q^{n-k}$. That is, the probability that $a$ is supported is $1-q^{n-k}$. As there are $k$ atoms in $S$, the probability that every atom in $S$ is supported by a pure rule in $P$ is  $(1-q^{n-k})^k$.

Combining the above three conditions, we know that the probability that $S$ is an answer set of random program $P$ is as follows.
\[
\Pr(k)=q^{(n-k)(n-k-1)}(1-q^{n-k})^k(1-d)^{n-k}.
\]

\end{proof}

%

Now we are ready to present the main result in this section, which shows that the average number of answer sets for random logic programs generated under the linear model converges to a constant when the number of atoms approaches infinity.
This constant is determined by $c_1$ and $c_2$, e.~g., when $c_1=5$ and $c_2=0$, the constant is around $1.6$.  
%
\begin{theorem}\label{thm:mainthm}
Let $P$ denote a random program generated under the linear model $\linear{c_1}{c_2}$ and $E[|\AS{P}|]$ be the expected number of answer sets for random program $P$. Then 
\begin{equation}
\label{eqn:EASP}
\lim\limits_{n \to \infty} E[|\AS{P}|]  
= \frac{\alpha e^{\frac{c_1-c_2}{\alpha}}}{\alpha+c_1},
\end{equation}
where $\alpha >1$ is the unique solution of the equation $\ln{\alpha}=c_1/\alpha$.
\end{theorem}

This result gives an estimation for the average number of answer sets for a random program. Before we prove Theorem~\ref{thm:mainthm}, let us look at its application in predicting the consistency of a random program. 

For a random program $P$ and a set of atoms $S$, by $e_S$ we denote the (probabilistic) event that a given set of atoms is an answer set for $P$. We introduce the following property for random programs:

\noindent {\bf (ASI)}\quad Given a random program $P$, $\Pr(e_S|e_{S'}) = \Pr(e_S)$ for any two sets $S$ and $S'$ of atoms. 

The `I' in {\bf (ASI)} is for `Independence'. Informally, the above property says that for any two sets of atoms $S$ and $S'$, the events $e_S$ and $e_{S'}$ are independent of each other. We remark that this property does not hold in general. For example, suppose $S_1 \subset S_2 \subset A_n$. If $S_1$ is an answer set of $P$, then $S_2$ must not be an answer set of $P$. This implies that $e_{S_1}$ and $e_{S_2}$ are actually not independent. However, when the set of atoms $A_n$ is sufficiently large, by Theorem~\ref{thm:mainthm}, the average number of answer sets will be relatively small compared to the number of all subsets of $A_n$. As a result, there will be a relatively small number of pairs $S\subseteq A_n$ and $S'\subseteq A_n$ with $S\neq S'$ such that $e_S$ and $e_{S'}$ are not independent. Thus, when $n$ is sufficiently large, the impact of dependency for answer sets will be not radical. Under the {\bf (ASI)} assumption, we are able to derive an estimation for the probability that a random program has an answer set. 
%
\begin{proposition} \label{pro:consistency}
Let $P$ be a random program on a set $A_n$ of $n$ atoms, generated under $\linear{c_1}{c_2}$, with $n>\max(c_1,c_2)$.
If {\bf (ASI)} holds and $n$ is sufficiently large, then 
\begin{equation}
\label{eq:exp_to_pro}
\Pr(E(|\AS{P}|>0))\approx 1-e^{-E(|\AS{P}|)}.
\end{equation}
\end{proposition}
%
As explained, {\bf (ASI)} does not hold in realistic situation. Our experiments indeed show that there is a shift between the estimated probability determined by Eq.(\ref{pro:consistency}) and the actual probability. However, The experimental results suggest that this shift can be remedied by applying a factor $\gamma$ of around $0.5$ to $E(|\AS{P}|)$ in Eq.(\ref{eq:exp_to_pro}), see Section~\ref{sec:experiment} for details.
So, combining Theorem~\ref{thm:mainthm} and Proposition~\ref{pro:consistency}, we will be able to estimate the probability for the consistency of random programs. 
\begin{proof}
Let $e_{S,k}$ be the event that $S\subset A_n$ is an answer set of size $k$ for random program $P$.
We first observe that by Eq.(\ref{eq:Pr:k}), $\lim\limits_{n \rightarrow \infty}\Pr(e_{S,k}) = 0$.
Recall that $\AS{P,k}$ is the set of answer sets of size $k$ for logic program $P$.

If $n$ is sufficiently large, then

\[
\begin{split}
\Pr(E(|\AS{P}|)>0)&=1-\Pr(E(|\AS{P}|)=0)\\
&=1-\prod_{0< k< n}{\Pr(E(|\AS{P,k}|)=0)}\\
&=1-\prod_{0< k< n}{[1-\Pr(e_{S,k})]^{{n \choose k}}} \\ 
&=1-\prod_{0< k< n}{[1-\Pr(e_{S,k})]^{\frac{1}{\Pr(e_{S,k})}\cdot \Pr(e_{S,k})\cdot {n \choose k}}}  \\
&\approx 1-\prod_{0< k< n}{e^{-\Pr(e_{S,k})\times {n \choose k}}}, \quad \mathrm{because~}\lim\limits_{x \rightarrow 0}(1-x)^{\frac{1}{x}} = e^{-1} \\
&=1-\prod_{0< k< n}{e^{-E(|\AS{P,k}|)}}\\
&=1-e^{-\sum_{0< k< n}{E(|\AS{P,k}|)}}\\
&=1-e^{-E(|\AS{P}|)}.
\end{split}
\] 
\end{proof}

%% file: proof.tex
In the rest of this section, we will present a formal proof of Theorem~\ref{thm:mainthm}.
Let us first outline a sketch for the proof. 
In order to prove Eq.(\ref{eqn:EASP}), our first goal will be to show
that $E[|\AS{P}|]$ is the sum of $E[N_k]$'s for $0< k < n$.

For an integer $k$ with $0< k < n$, we use $\AS{P,k}$ to denote the collection of answer sets of size $k$ for program $P$,
i.e.,
$\AS{P,k} = \{S \;|\; S\in \AS{P},|S|=k\}$. Then the number $N_k = |\AS{P,k}|$ is a random variable. 
It is easy to see that the expected number of answer sets of size $k$ for random program $P$ is
\begin{equation}\label{eq:E:Nk}
E[N_k] = {n \choose k}\Pr(k).
\end{equation}

So the expected (total) number of answer sets for $P$, denoted $E[|\AS{P}|]$, can be expressed as
\begin{equation}\label{eq:exp:as}
E[|\AS{P}|] = \sum\limits_{k=1}^{n-1} E[N_k].
\end{equation}

Note that by Proposition~\ref{pro:no:0:n}, a random program generated under the linear model has neither answer sets of size $0$ nor $n$. So, we can ignore the cases of $k=0$ and $k=n$. 

Our next goal is, based on Eq.(\ref{eq:exp:as}), to show that 
\begin{equation}\label{eq:tmp:1}
\lim\limits_{n \to \infty}E[|\AS{P}|] =  \lim\limits_{n \to \infty} \int^{n}_{1} \phi(x) \mathrm{d}x.
\end{equation}
where the function $\phi(x)$ is defined by
\begin{equation}\label{eq:def:phi_x}
\phi(x)=\sqrt{\frac{n}{2\pi x(n-x)}}
\left(\frac{n(1-q^{n-x})}{x}\right)^{x}\left(\frac{nrq^{n-x}}{n-x}\right)^{n-x}.
\end{equation}

At the same time, we are going to show that 
\begin{equation}\label{eq:E:phi}
\lim\limits_{n \to \infty} \int^{n}_{1} \phi(x) \mathrm{d}x
= \lim\limits_{n \to \infty} \int^{\infty}_{-\infty} \chi(x) \mathrm{d}x,
\end{equation}
where the function $\chi(x)$ defined below is a normal distribution function multiplied by a constant.

Thus, it follows from Eq.(\ref{eq:tmp:1}) and Eq.(\ref{eq:E:phi}) that
\[
  \lim\limits_{n \to \infty} E[|\AS{P}|] = \lim\limits_{n \to \infty} \int^{\infty}_{-\infty} \chi(x) \mathrm{d}x.
\]
As the above integral of $\chi(x)$ is $\alpha e^{\frac{c_1-c_2}{\alpha}}/(\alpha+c_1)$, which can be figured out easily, the conclusion of Theorem~\ref{thm:mainthm} will be proven. 

Here $\alpha>1$ is the unique solution of the equation $\alpha^\alpha = e^{c_1}$
and $\chi(x)$ is the normal distribution function 
$$\mathcal{N}_{x_0,\sigma}(x) = \frac{1}{\sqrt{2\pi}\sigma}e^{-\frac{(x-x_0)^2}{2\sigma^2}}$$
multiplied by a constant $\sqrt{2\pi}\sigma\phi(x_0)$:
\begin{equation}\label{eq:def:chi_x}
\chi(x)=\left(\sqrt{2\pi}\sigma\phi(x_0)\right)\mathcal{N}_{x_0,\sigma}(x) =\phi(x_0)e^{-\frac{(x-x_0)^2}{2\sigma^2}}.
\end{equation}
while $x_0$ and $\sigma$ are defined, respectively, as follows.
\begin{equation}\label{eq:def:x0}
x_0=\frac{(\alpha-1)n}{\alpha}.
\end{equation}
\begin{equation}\label{eq:def:sigma}
\sigma=\frac{\sqrt{(\alpha-1)n}}{\alpha + c_1}.
\end{equation}

Some remarks are in order. As $c_1>0$, if $\alpha^{\alpha}=e^{c_1}$ for some $\alpha$, it must be the case that $\alpha > 1$. On the other hand, if $\alpha>1$, the function $\alpha^{\alpha}$ is monotonically increasing and thus the equation $\alpha^{\alpha}=e^{c_1}$ must have a unique solution.

Moreover, we define
\begin{equation}\label{eq:def:c0}
c_0=\max(\frac{\sqrt{2}(\alpha + c_1)}{\sqrt{\alpha -1}},\frac{1}{\sqrt{c_1}}).
\end{equation}
\begin{equation}\label{eq:def:Delta}
\Delta=c_0\sqrt{n\ln{n}}.
\end{equation}

Before providing the proof of Theorem~\ref{thm:mainthm}, we first prove some technical results.

The following result shows that $\phi(k)$, as defined in Eq.(\ref{eq:def:phi_x}), is indeed a tight approximation to $E[N_k]$.

\begin{proposition}\label{pro:E_k}
Let $P$ be a random program on a set $A_n$ of $n$ atoms, generated under $\linear{c_1}{c_2}$, with $n>\max(c_1,c_2)$.
Let $E[N_k]$ be the expected number of answer sets of size $k$ for $P$ ($0< k < n$). Then ,
\begin{equation}
\label{eq:E_k1}
\frac{4\pi^2}{e^2}\phi(k) \leq E[N_k] \leq \frac{e}{2\pi}\phi(k).
\end{equation}
\begin{equation}
\label{eq:E_k2}
E[N_k] = \phi(k)\left(1 + O\left(\frac{1}{\min(k,n-k)}\right)\right).
\end{equation}
\end{proposition}
%
\begin{proof}
Note that
\[
E[N_k]={n \choose k}\Pr(k) = \frac{n!}{k!(n-k)!} \Pr(k).
\]
By Proposition~\ref{pro:prob:as},
\[
E[N_k]=\frac{n!}{k!(n-k)!}q^{(n-k)(n-k-1)}(1-q^{n-k})^k(1-d)^{n-k}.
\]
Let $r=(1-d)/(1-p)=(1-d)/q$. Then
\[
E[N_k]=\frac{n!}{k!(n-k)!}(q^{n-k}r)^{n-k}(1-q^{n-k})^k.
\]

Applying Stirling's approximation to $n!$, $k!$ and $(n-k)!$, and based on the two properties of Stirling's approximation presented in Section~\ref{sec:math_review}, Eq.(\ref{eq:E_k1}) and Eq.(\ref{eq:E_k2}) are obtained. 
\end{proof}
%

By Proposition~\ref{pro:E_k}, we can show the following result.

\begin{proposition}\label{pro:sum}
Let $P$ be a random program on a set $A_n$ of $n$ atoms, generated under $\linear{c_1}{c_2}$, with $n>\max(c_1,c_2)$.
If $E[N_k]$ and $\phi(k)$ are defined as in Eq.(\ref{eq:E:Nk}) and Eq.(\ref{eq:def:phi_x}, then
\begin{equation}
\lim\limits_{n \to \infty}\sum\limits_{k=1}^{n-1} E[N_k] = \lim\limits_{n \to \infty}\sum\limits_{k=1}^{n-1} \phi(k). 
\end{equation}
\end{proposition}

\begin{proof}
Let $\Delta$ be defined as in Eq.(\ref{eq:def:Delta}). Then
\[
\sum\limits_{k=1}^{n-1} E(N_k) = \sum\limits_{k=1}^{\lfloor{x_0-\Delta\rfloor}} E(N_k) + \sum\limits_{k=\lfloor{x_0-\Delta}\rfloor+1}^{\lfloor{x_0+\Delta\rfloor}-1} E(N_k)+
\sum\limits_{k=\lfloor{x_0+\Delta\rfloor}}^{n-1} E(N_k).
\]
By inequality~(\ref{eq:E_k1}),
\[
\sum\limits_{k=1}^{\lfloor{x_0-\Delta\rfloor}} E(N_k) + 
\sum\limits_{k=\lfloor{x_0+\Delta\rfloor}}^{n-1} E(N_k)  \leq \frac{e}{2\pi} 
\left( \sum\limits_{k=1}^{\lfloor{x_0-\Delta\rfloor}} \phi(k) + 
\sum\limits_{k=\lfloor{x_0+\Delta\rfloor}}^{n-1} \phi(k)  \right).
\]

By Lemma~\ref{lem:phi:small_at_edge} (in Section~\ref{sec:lemma}), 
\[
\lim\limits_{n \to \infty}\left(\sum\limits_{k=1}^{\lfloor{x_0-\Delta\rfloor}} \phi(k) + 
\sum\limits_{k=\lfloor{x_0+\Delta\rfloor}}^{n-1} \phi(k)\right) =0.
\]


Based on Eq.(\ref{eq:E_k2}), and the fact that both $\phi(k)$ and $E(N_k)$ are non-negative,

\[
\begin{split}
&\lim\limits_{n \to \infty}\left(\sum\limits_{k=1}^{\lfloor{x_0-\Delta\rfloor}} E(N_k) + 
\sum\limits_{k=\lfloor{x_0+\Delta\rfloor}}^{n-1} E(N_k)\right) =\\
&\lim\limits_{n \to \infty}\left(\sum\limits_{k=1}^{\lfloor{x_0-\Delta\rfloor}} \phi(k)\left(1 + O\left(\frac{1}{\min(k,n-k)}\right)\right)+ 
\sum\limits_{k=\lfloor{x_0+\Delta\rfloor}}^{n-1} \phi(k)\left(1 + O\left(\frac{1}{\min(k,n-k)}\right)\right)\right)\\
&\leq \lim\limits_{n \to \infty}\left(\sum\limits_{k=1}^{\lfloor{x_0-\Delta\rfloor}} \phi(k)+ 
\sum\limits_{k=\lfloor{x_0+\Delta\rfloor}}^{n-1} \phi(k)\right)\left(1 + O(1)\right)=0.
\end{split}
\]


As $ E(N_k) \geq 0$ for $k\ge 1$, we have that  
\[
\lim\limits_{n \to \infty}\left(\sum\limits_{k=1}^{\lfloor{x_0-\Delta\rfloor}} E(N_k) + 
\sum\limits_{k=\lfloor{x_0+\Delta\rfloor}}^{n-1} E(N_k)\right) =0.
\]

By Eq.(\ref{eq:E_k2}),
\[
\begin{split}
\sum\limits_{k=\lfloor{x_0-\Delta}\rfloor+1}^{\lfloor{x_0+\Delta\rfloor}-1} E(N_k) &= 
\sum\limits_{k=\lfloor{x_0-\Delta}\rfloor+1}^{\lfloor{x_0+\Delta\rfloor}-1}\left( \phi(k) 
\left(1 + O\left(\frac{1}{\min(k,n-k)}\right)\right)\right)\\
&= \left(\sum\limits_{k=\lfloor{x_0-\Delta}\rfloor+1}^{\lfloor{x_0+\Delta\rfloor}-1} \phi(k)\right) 
\left(1 + O\left(\frac{1}{\min(x_0-\Delta,n-x_0-\Delta)}\right)\right).
\end{split}
\]
By Eq.(\ref{eq:def:x0}) and Eq.(\ref{eq:def:Delta}), we have that
\[
\lim\limits_{n \to \infty}{\min(x_0-\Delta,n-x_0-\Delta)}=\infty.
\]
So
\[
\lim\limits_{n \to \infty}\sum\limits_{k=\lfloor{x_0-\Delta}\rfloor+1}^{\lfloor{x_0+\Delta\rfloor}-1} E(N_k) = 
\lim\limits_{n \to \infty}\sum\limits_{k=\lfloor{x_0-\Delta}\rfloor+1}^{\lfloor{x_0+\Delta\rfloor}-1} \phi(k).
\]
Therefore, the conclusion is proved.
\end{proof}

The next result shows that the integral of $\phi(x)$ can be obtained through the integral of $\chi(x)$,
which is useful as the integral of $\chi(x)$ can be easily figured out.
%
\begin{proposition} \label{pro:phi:chi}
Let $P$ be a random program on a set $A_n$ of $n$ atoms, generated under $\linear{c_1}{c_2}$, with $n>\max(c_1,c_2)$.
If the continuous functions $\phi(x)$ and $\chi(x)$ are defined as in Eq.(\ref{eq:def:phi_x}) and Eq.(\ref{eq:def:chi_x}), then
\begin{equation}\label{eq:phi:chi}
\lim\limits_{n \to \infty} \int^{n}_{1} \phi(x) \mathrm{d}x= \lim\limits_{n \to \infty} \int^{-\infty}_{\infty} \chi(x) \mathrm{d}x.
\end{equation}
\end{proposition}
\begin{proof}
Let $\Delta=c_0\sqrt{n\ln{n}}$ be defined as in Eq.(\ref{eq:def:Delta}). By Lemma~\ref{lem:phi:small_at_edge}, it follows that
\[
\lim\limits_{n \to \infty}\left(\int^{x_0-\Delta}_1 \phi(x)\mathrm{d}x + \int_{x_0+\Delta}^n \phi(x)\mathrm{d}x\right)= 0.
\]

By Lemma~\ref{lem:cond1} and Lemma~\ref{lem:cond2}, Eq.(\ref{eq:phi:chi}) holds. 
\end{proof}

Now we are ready to present the proof of Theorem~\ref{thm:mainthm}, a main result in this paper, 
%
%
\begin{proof}[Proof of Theorem~\ref{thm:mainthm}]
Given a random program $P$, the expected total number of answer sets for $P$ is
\[
 E[|\AS{P}|] =  \sum\limits_{k=1}^{n-1} E[N_k].
\]
By Proposition~\ref{pro:sum}
and Proposition~\ref{pro:phi:chi},
\[
\begin{split}
\lim\limits_{n \to \infty} E[|\AS{P}|] &=\lim\limits_{n \to \infty} \sum\limits_{k=1}^{n-1} E[N_k]
= \lim\limits_{n \to \infty}\sum\limits_{k=1}^{n-1} \phi(k)\\
&= \lim\limits_{n \to \infty} \int^{n}_{1} \phi(x) \mathrm{d}x
= \lim\limits_{n \to \infty} \int^{\infty}_{-\infty} \chi(x) \mathrm{d}x.\\
\end{split}
\]
Then 
\[
\lim\limits_{n \to \infty}\int^{\infty}_{-\infty} \chi(x) \mathrm{d}x
=  \lim\limits_{n \to \infty} \sqrt{2\pi}\sigma\phi(x_0) 
=\frac{\alpha e^{\frac{c_1-c_2}{\alpha}}}{\alpha+c_1}.
\]
Therefore,
\[
\lim\limits_{n \to \infty} E[|\AS{P}|] = \frac{\alpha e^{\frac{c_1-c_2}{\alpha}}}{\alpha+c_1}.
\]
\end{proof}


%% file: experiment.tex
\section{Experimental Results} \label{sec:experiment}

In this section, we describe some experimental results about the average number of answer sets, the size distribution of answer sets,  and the probability of consistency for random programs under the linear model. For the average number of answer sets, our experimental results closely match the theoretical results obtained in Section~\ref{sec:theorem}.

To conduct the experiments, we have developed a software tool to generate random logic programs, which is able to randomly generate logic programs based on the user-input parameters, such as the type of programs, the number of atoms, the number of literals in a rule, the number of rules in a program and the number of programs in a test set etc. After a set of random programs are generated, the tool invokes an ASP solver to compute the answer sets of the random programs, records the test results in a file, and analyses them. The experimental results in this section were based on the ASP solver clasp \cite{GebserKS09}, but same patterns were obtained for test cases on which dlv \cite{LeonePFEGPS06} and smodels \cite{SyrjanenN01} were also used.

We have conducted a significant number of experiments to corroborate 
the theoretical results obtained in Section~\ref{sec:theorem} including Theorem~\ref{thm:mainthm}. In
order to get a feel for how quickly the experimental distribution converges to the theoretical one, we tested the difference rate of these two values for varied numbers of atoms. 
The experimental results show that the theorem can be used to predict practical
situations.
Some other statistical properties of random programs generated under the linear model were also experimented, such as the size distribution of answer sets. 
Positive results are received for nearly all of our experiments. 
In this section, we report the results from two of our experiments. In the first experiment, we set $c_2=0$, which means there are no contradiction rules in the programs. In the second experiment, we set $c_2$ from $0$ to 20 to test the impact of contradiction rules on the random programs. 

\subsection{Experiment 1: Random Programs without contradiction rules}

In this experiment, $c_1=5$, $c_2=0$, and $n$ varies with values $50, 100, 150, ... , 500$, respectively. For each of these values of $n$, $5,000$ logic programs were randomly and independently generated under the linear model. 

Given that $c_1=5$ and $\alpha>1$ is determined by $\alpha^\alpha = e^{c_1}$,
we have that $\alpha \approx 3.7687$. Thus, by Eq.(\ref{eqn:EASP}), it follows that $E(|\AS{P}|) \approx 1.6274$.

We use $N_{Avg}$ to denote the average number of answer sets for the $5,000$ programs in each test generated under the linear model.
The (experimental) values for $N_{Avg}$ and their corresponding theoretical values (i.e., the expected number $E[N_k]$ of answer sets for random programs determined by Eq.(\ref{eqn:EASP})) are listed in Table~{\ref{tab:exp1}}. The experimental and theoretical results are visualized in Figure~\ref{fig:exp1}.
We can see that these two values are very close even if $n$ is relatively small.

\begin{table}[h]
\caption{The average number and expected number of answer sets for random programs when $c_1=5$ and $c_2=0$.}
\label{tab:exp1}
    \begin{tabular}{ | c |c |c| c |c |c |c | c | c | c |}
     \hline
$n$ & $N_{Avg}$ & $n$ & $N_{Avg}$ & $n$ & $N_{Avg}$ & $n$ & $N_{Avg}$  & $n$ & $N_{Avg}$ \\  \hline
50 & 1.6404 & 100 & 1.6674 &150&1.6334&200&1.5794&250&1.6874\\ \hline
300&1.6178&350&1.6738&400&1.5672&450&1.682&500&1.632\\ \hline
\end{tabular}
 \end{table}

\begin{figure}[h]
\centering
\includegraphics[height=60mm]{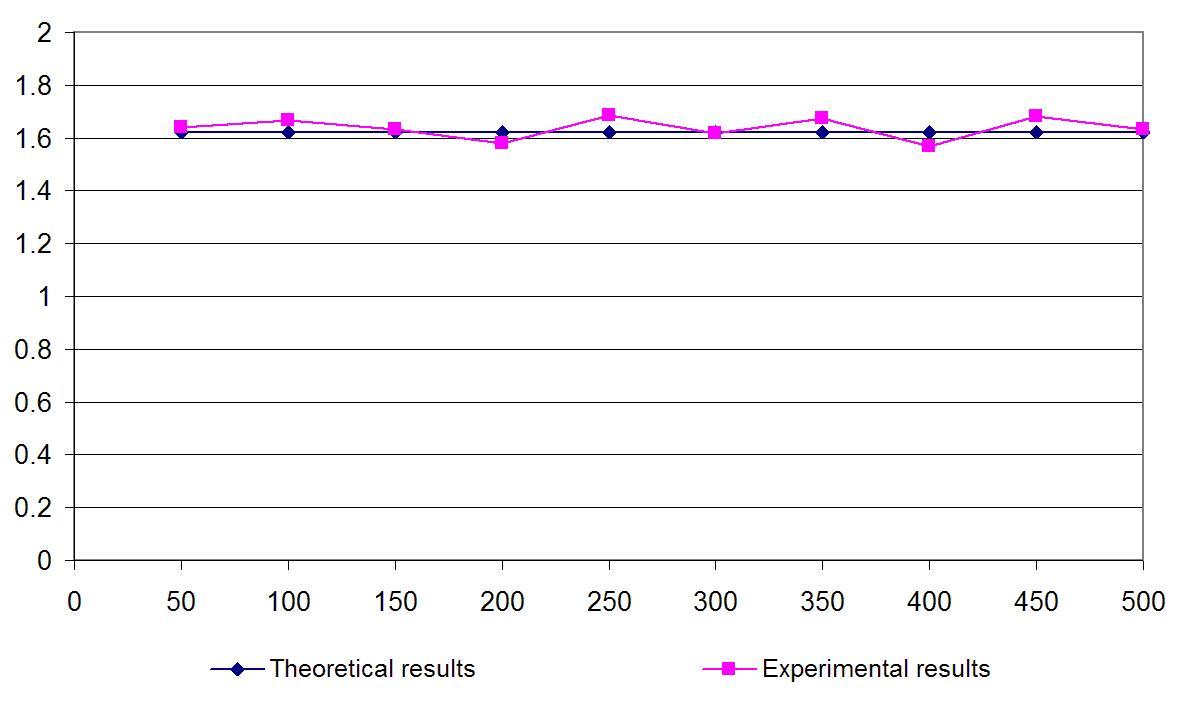}
\caption{The average number and expected number of answer sets for random programs when $c_1=5$ and $c_2=0$ ($x$-axis: the number of atoms; $y$-axis: the average number of the answer sets.}
\label{fig:exp1}
\end{figure}

Another important result obtained from this experiment is about the size distribution of answer sets for random programs. Specifically, the experiment supports a conjecture that the distribution of the average size of answer sets for random programs obeys a normal distribution. 

The experimental result can be easily seen by comparing the following three types of values for $0\le k\le n$, which are visualized as three curves in Figure~\ref{fig:50_c5} and Figure~\ref{fig:exp12} with $n=50$ and $n=400$ respectively.


\noindent \emph{Average number of answer sets for the $5,000$ programs randomly generated in each test} (referred to as `Experiment Result' in Figure~\ref{fig:50_c5} and Figure~\ref{fig:exp12}): We took $n=50, 100, ..., 500$, respectively, and for each of these values of $n$, we randomly generated $5,000$ programs under the linear model. For each $k$ ($0\le k\le n$), we calculated the average number of answer sets of size $k$ for these programs, i.e., the ratio of the total number of answer sets of size $k$ for all these programs divided by $5,000$. 

\noindent \emph{Expected number of answer sets for random programs under the linear model} (referred to as `The Model' in Figure~\ref{fig:50_c5} and Figure~\ref{fig:exp12}): In order to compare the experimental values with their theoretical counterparts, for each $0\le k\le n$, we calculated the expected number $E[N_k]$ of answer sets of size $k$ for random programs under the linear model. 

\noindent \emph{Normal Distribution function}: 
\quad The above two types of values were also compared with 
the function $\chi(x)$ defined by Eq.(\ref{eq:def:chi_x}), which is actually the normal distribution function $\mathcal{N}(x_0,\sigma)$ multiplied by a constant.



\begin{figure}[h]
\centering
\includegraphics[height=60mm]{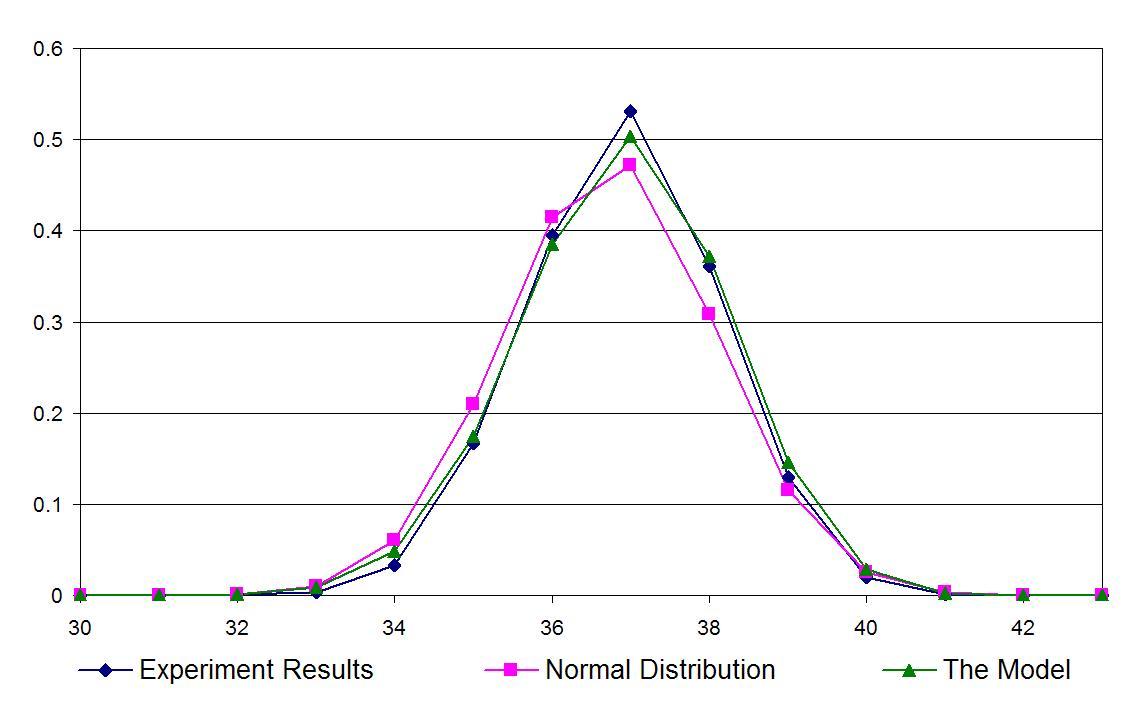}
\caption{Three distributions of answer sets for $c_1=5$, $c_2=0$, and $n=50$ ($5,000$ programs are generated; $x$-axis: the size of the answer sets; $y$-axis: the average number of answer sets of a given size).} 
\label{fig:50_c5}
\end{figure}

\begin{figure}[h]
\centering
\includegraphics[height=60mm]{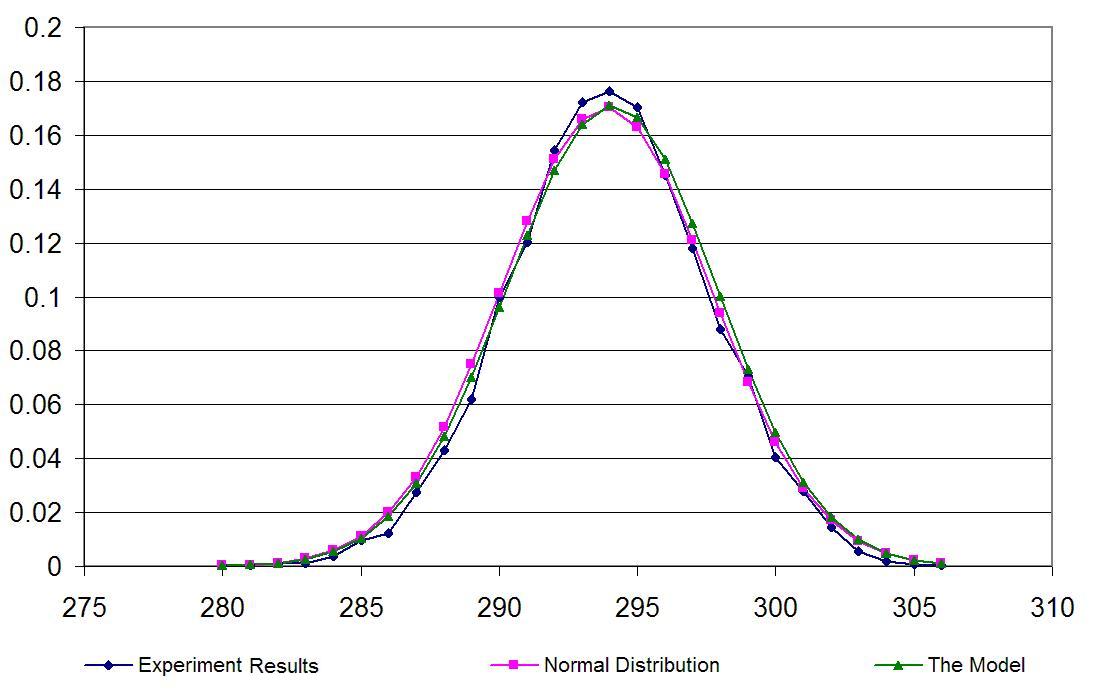}
\caption{Three distributions of answer sets for $c_1=5$, $c_2=0$, and $n=400$ ($5,000$ programs are generated; $x$-axis: the size of the answer sets; $y$-axis: the average number of answer sets of a given size).} 
\label{fig:exp12}
\end{figure}

Figure~\ref{fig:50_c5} and Figure~\ref{fig:exp12} show that even for relatively small values of $n$, the theoretical results are still very close to the experimental results. 
In order to see how quickly the experimental distribution converges to the theoretical one,
we consider the rate variance function $D$: For two discrete functions $f$ and $g$ on the interval $[1,n-1]$ with $f(k)>0$ ($1\le k \le n-1$), we define
\[
D(f,g)=\frac{\sum_{k=1}^{n-1}(f(k)-g(k))^2}{\sum_{k=1}^{n-1}f(k)^2}.
\]
Clearly, the closer $f$ and $g$, the smaller $D(f,g)$, and vice versa. The function $D(f,g)$ is often used in measuring the gap between two discrete functions $f$ and $g$. If we take $f$ as the normal distribution function and $g$ as the experimental distribution function (i.e., the average size of answer sets based on the $5,000$ programs randomly generated in each test). The resulting rate variance function is depicted in Figure~\ref{fig:difrate}.
This diagram shows that, as $n$ increases, the rate variance gradually decreases. It also shows that the rate variance is very small even when $n=50$. This experimental result further suggests the conjecture that the size distribution of answer sets obeys a normal distribution.



\begin{figure}[h]
\centering
\includegraphics[height=50mm]{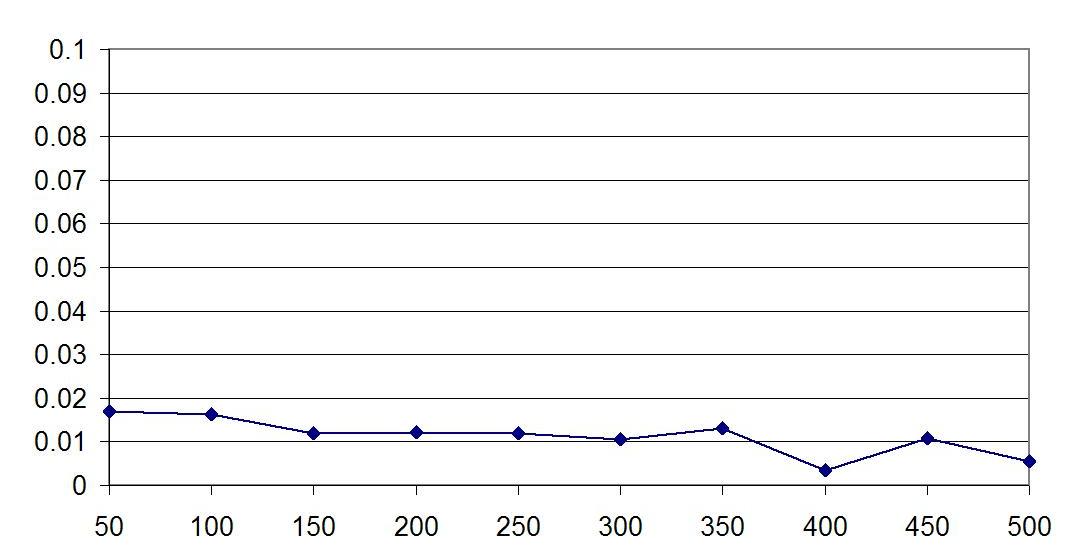}
\caption{The difference rate between the normal distribution and the experimental results of the answer set distribution with $c_1=5$, $c_2=0$ ($5,000$ programs are generated with each testing point; $x$-axis: the number of atoms; $y$-axis: difference rate).} 
\label{fig:difrate}
\end{figure}

%



\subsection{Experiment 2: Random Programs with contradiction rules}

In this experiment, we tested random programs that may contain contradiction rules and obtained similar experimental results as in the first experiment. We set $c_1=10$, $n=200$, and $c_2=0, 1, 2, ..., 20$, respectively. For each value of $c_2$, $5,000$ programs were independently generated under the linear model.

Given $c_1=10$, it follows by Eq.(\ref{eq:def:x0}) and Eq.(\ref{eq:def:sigma}) that $\alpha \approx 5.7289$, $x_0 \approx 165.0894$, and $\sigma \approx 1.9552$. The value $\phi(x_0)$,  which depends on $c_2$, decreases roughly from $0.4257$ (when $c_2=0$) to $0.01297$ (when $c_2=20$). 

On the other hand, based on Eq.(\ref{eqn:EASP}), we can figure out the expected number of answer sets for each $c_2$. 

These two types of values are visualized as two curves in Figure~\ref{fig:exp2}. It shows that these two curves are very close to each other, which means our theoretical result on size distribution of answer sets is corroborated by the experimental result. 


\begin{figure}[h]
\centering
\includegraphics[height=60mm]{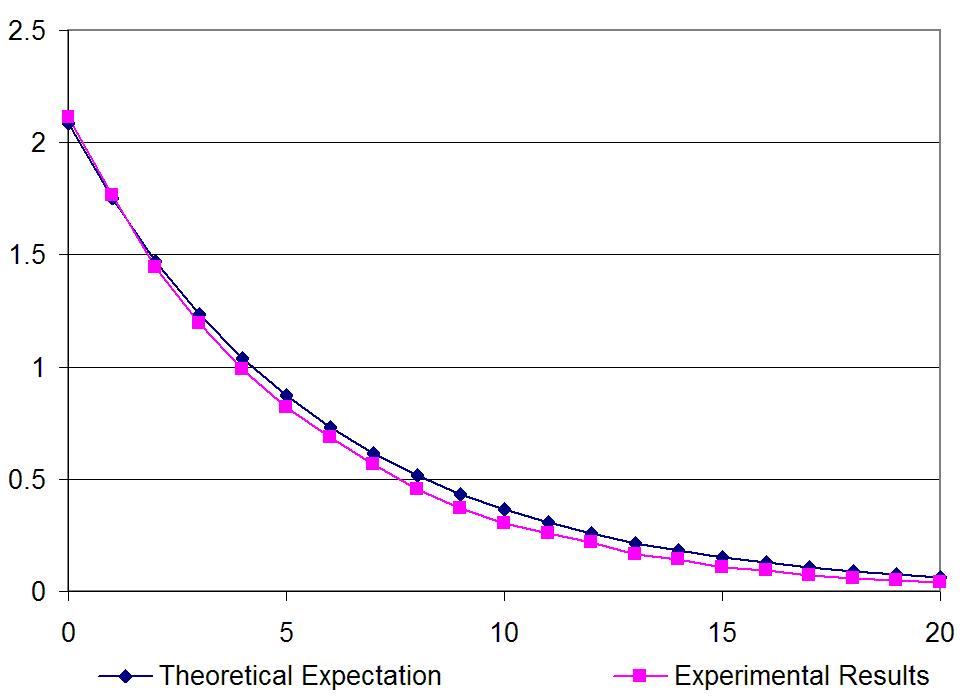}
\caption{The comparison of expected number and average number of answer sets for $c_1=10$, $n=200$, and $c_2=0, 1, ... , 20$. The $x$-axis is for the number of contradiction rules ($c_2$) and the $y$-axis is for the average number of answer sets.}
\label{fig:exp2}
\end{figure}
    
Similar to the first experiment, the size distribution of answer sets was also investigated experimentally. In this case, we took $c_2=4$ and three types of values were obtained (shown in Figure~\ref{fig:exp22}). There is a slight shift between the linear model and the normal distribution. We expect that when the number $n$ is sufficiently large, this shift will become narrower. For example, when $n$ increases from $200$ to $400$, the shift is significantly reduced.

\begin{figure}[h]
\centering
\includegraphics[height=60mm]{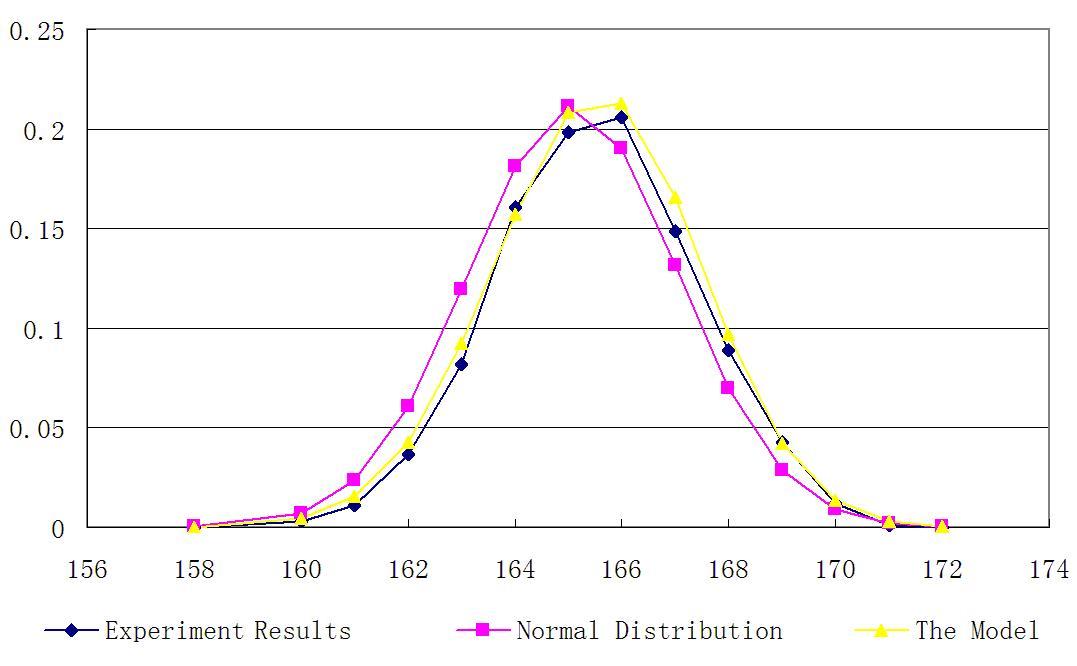}
\caption{The comparison of the distribution of answer set. $x$-axis is the size of the answer sets. $y$-axis is the average number of answer sets for a program of that given size. The first curve shows the experimental results. The second and the third curves are based on the theoretical estimation of $\chi(x)$ (normal distribution) and $\phi(x)$ (the model) respectively. ($c_1=10$, $c_2=4$, $n=200$, 5000 independent programs are used to get the experimental results)}
\label{fig:exp22}
\end{figure}

\subsection{Experiment 3: Approximating the probability for consistency of random programs}
\label{subsec:3}

In this subsection we present our experimental results on verifying the formula for predicting consistency of random programs (discussed in Section~\ref{sec:theorem}):
\begin{equation}
\label{eq:exp_to_pro:modified}
\Pr(E(|\AS{P}|>0))\approx 1-e^{-\gamma\cdot E(|\AS{P}|)}.
\end{equation}
Here $\gamma$ is a constant around $0.5$ (i.e. independent of $n$). We tested various pairs of $c_1$ and $c_2$. For each such pair, we took $n=100, 150, 200, ..., 1000$. Then for each value of $n$, we computed the value determined by Eq.(\ref{eq:exp_to_pro:modified}). For each value of $n$, we generated $5,000$ programs randomly and computed the ratio of consistent programs to all $5,000$ programs. 



Our experimental results corroborate the estimation in Eq.(\ref{eq:exp_to_pro:modified}). So this formula can be used to predict the consistency of random programs generated under the linear model. The corresponding values for two cases we tested are depicted in Figures~\ref{fig:pro_comp_adj_c1_3} and \ref{fig:pro_comp_adj_c1_4_c2_4}. In each figure, the upper curve is for the value determined by Eq.(\ref{eq:exp_to_pro}), the middle curve is for the ratio of consistent programs to all $5,000$ programs randomly generated, and the lower curve is for the value determined by Eq.(\ref{eq:exp_to_pro:modified}).

				
\begin{figure}[h]
\centering
\includegraphics[height=70mm]{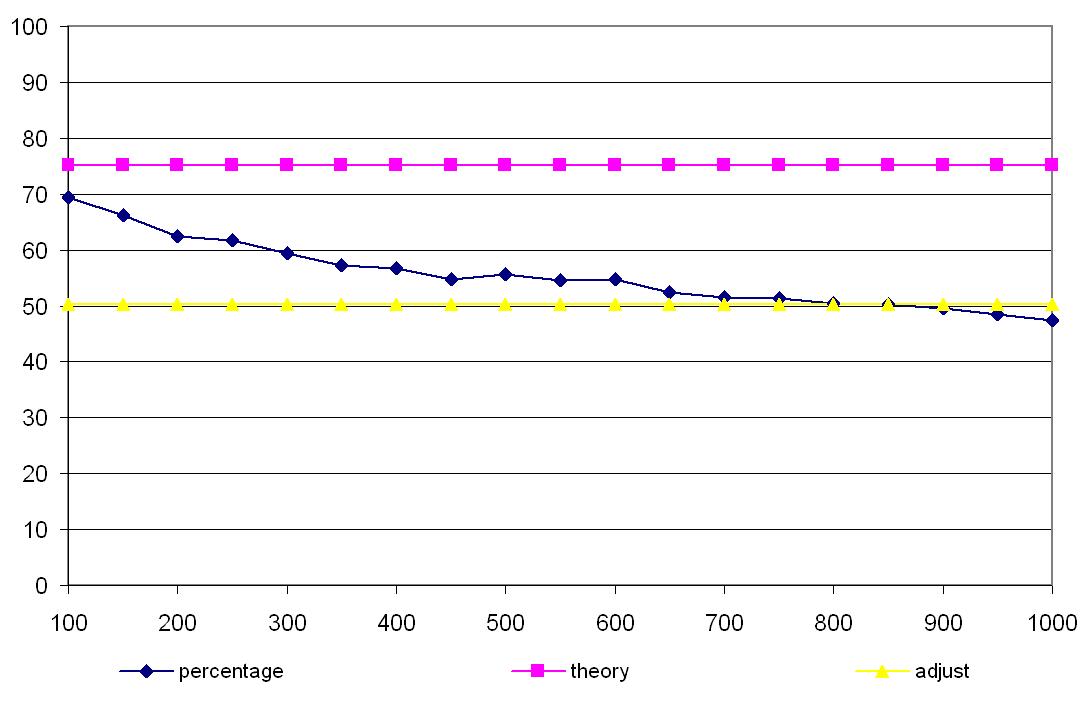}
\caption{Three values of estimating the probability for consistency of random programs when $c_1=3$ and $c_2=0$. $\gamma = 0.5$. The $x$-axis is for $n$, the number of atoms, the $y$-axis is the ratios of consistent programs to all $5,000$ programs.}
\label{fig:pro_comp_adj_c1_3}
\end{figure}



\begin{figure}[h]
\centering
\includegraphics[height=70mm]{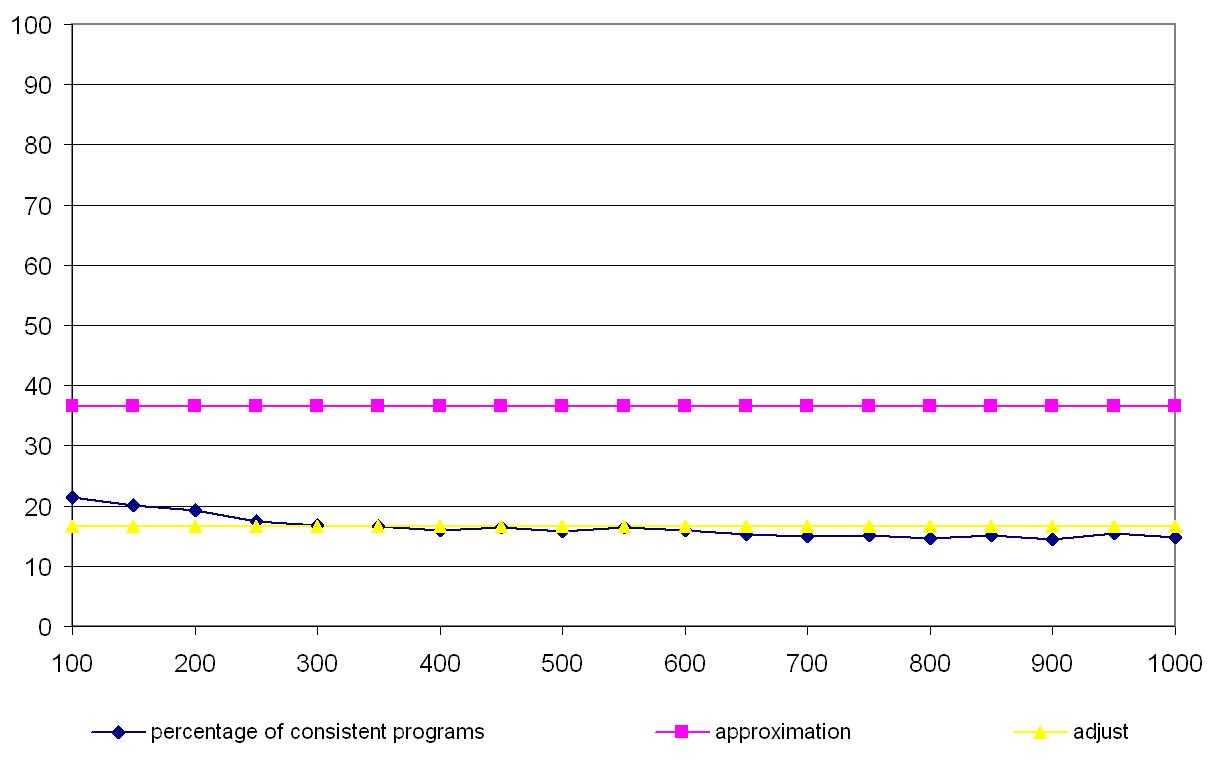}
\caption{Three values of estimating the probability for consistency of random programs when $c_1=4$ and $c_2=4$. $\gamma = 0.5$. The $x$-axis is for $n$, the number of atoms, the $y$-axis is the ratios of consistent programs to all $5,000$ programs.}
\label{fig:pro_comp_adj_c1_4_c2_4}
\end{figure}


%

%% file: conclusion.tex
\section{Conclusion} \label{sec:conclusion}

We have proposed a new model of randomly generating logic programs under answer set semantics, called \emph{linear model}.
The average size of a random program generated in this way is linear to the total number of atoms.
We have proved some mathematical results and the main result shows that the expected number of answer sets of random programs under the linear model converges to a constant that is determined by the probabilities of both pure rules and constraints. The formal proof of this result is mathematically involving as we have seen. The main result is further corroborated by our experiments. 
Another important experimental result reveals that the (size) distribution of answer sets for random programs generated under the linear model obeys a normal distribution. 

There are several issues for future work. First, it would be interesting to mathematically prove some results presented in Section~\ref{sec:experiment}. Second, it would be both interesting and useful to study phase transition phenomena for hardness. In this case, a new model for random programs may need to be designed based on an algorithm for ASP computation (for SAT and CSP, DPLL is often used for studying the hardness of random problems). Last, while the class of negative two-literal programs is of importance, it would be interesting to study properties of random logic programs that are more general than negative two-literal programs, such as the program classes discussed in \cite{Janhunen06,LoncT04}. However, it is not straightforward to carry over our proofs to those program classes. For instance, Proposition~\ref{pro:prob:as} may not hold for arbitrary two-literal programs.

\section*{Acknowledgement} The authors would like to thank the editor Michael Gelfond and three anonymous referees for their constructive comments, which helped significantly improve the quality of the paper. Thanks to Fangzhen Lin and Yi-Dong Shen for discussions on this topic. This work was supported by the Australian Research Council (ARC) under grants DP1093652 and DP130102302.

%% file: appendix.tex
\newenvironment{appthm}[2]
{\noindent{\bf#1~\ref{#2}}\begin{em}}{\par\end{em}\medskip}

\vskip 5mm

\section*{Appendix}

\subsection{Proofs for Section~\ref{sec:preliminaries}}

\begin{appthm}{Proposition}{thm:simple:form}
\translationTwo
\end{appthm}


\begin{proof}
First, it has been proven that each normal logic program
is equivalent to a negative logic program under answer set semantics
\cite{BrassD99,WangZ05}. So, without loss of generality, we assume that $P$ is a negative normal program.

Next, we show that each negative normal program $P$ can be transformed into a logic program
that consists of only two-literal rules and fact rules.
In fact, we can define the translation as follows.

For each rule $R$ in $P$ of the form $a \leftarrow
\naf c_1, \ldots, \naf c_n$ ($n\ge 0$), $R$ is replaced with the
following $n+1$ rules:

$a \la \naf e_R.$

$e_R \la c_1.$

$\ldots\ldots$

$e_R \la c_n.$

Here $e_R$ is a new atom introduced for the rule $R$. The resulting
logic program, denoted $\simple{P}$, is exactly a logic program
that consists of only two-literal rules and fact rules. We use $E_P$ to denote the set
of new atoms $e_R$ introduced above, that is, $E_P = \{e_R \D R\in
P\}$.

Note that, by applying unfolding transformation, $\simple{P}$ can be easily transformed
into a logic program that consists of only negative two-literal rules and fact rules.
As explained in Section~\ref{sec:preliminaries}, each rule can be expressed as a negative two-literal
rule by introducing a new atom. Thus, $\simple{P}$ is equivalent to a negative two-literal
program under answer set semantics.

So, it is sufficient to show that $\simple{P}$ and $P$ are indeed equivalent under answer
set semantics.

\noindent (1)\quad Let $S$ be an answer set of $P$. Take $S_e =
\{e_R\in E_P \D R\in P, \nbody{R}\cap S \neq \emptyset\}$. Then we show that
$S' = S\cup S_e$ is an answer set of $P'$. It suffices to prove that
$S'$ is a minimal model of $(P')^{S'}$.

By the definition of $S_e$, $S'$ is a model of $(P')^{S'}$.

We need only to show that $S'$ is minimal. Assume that there exists
$T'$ such that $T'\subseteq S'$ and $T'$ is also a model of $(P')^{S'}$.
Let $T = T'\setminus E_P$. Then $T$ is a model of $P^S$. To see
this, for each rule $R$ of the form $a \leftarrow \naf c_1, \ldots,
\naf c_t$ such that $c_i\not\in S$ for $i=1,\ldots,t$, if $a\not\in T$, then $a\not\in T'$. Thus $e_R\in T'$ by
$T'\models R$, which implies that $c_i\in T'$ for some $i$ ($1\le
i\le n$). So we have $c_i\in T$, that is, $T\models R$. By the
minimality of $S$, $T=S$.

Also, if $e_R\in S'$, then $c_i\in S$ for some $i$ ($1\le i\le n$).
This means $c_i\in T$ because $S=T$, which implies that $e_R\in T'$.
Therefore, $T'=S'$.

\noindent (2)\quad If $S'$ is an answer set of $P'$, we want to show
that $S=S'\setminus E_P$ is an answer set of $P$.

$S\models P^S$: for each rule $R$ of the form $a \leftarrow \naf
c_1, \ldots, \naf c_n$, if $R^+\in P^S$, then
$\{c_1,\ldots,c_n\}\cap S = \emptyset$, which implies that
$e_R\not\in S'$. Thus $R^+ = (a\la \naf e_R)^+ \in (P')^{S'}$. By
the assumption, $a\in S'$, that is, $a\in S$. Thus $S\models R^+$.

$S$ is a minimal model of $P^S$: Suppose that $T\subseteq S$ and
$T\models P^S$. Take $T' = T\cup S_e$ where $S_e = \{e_R\in E_P \D
\nbody{R}\cap S \neq \emptyset\}$. We first show that $T'\models
(P')^{S'}$.

Let $R'\in P'$ with $(R')^+\in (P')^{S'}$. Consider two possible
cases:

\noindent Case 1. $R'$ is of the form $a\la \naf e_R$: Then
$e_R\not\in S'$. By $T'\subseteq S'$, $e_R\not\in T'$. Then
$\{c_1,\ldots,c_n\}\cap S =\emptyset$. This means $R^+\in P^S$. Since
$T\models P^S$, we have $a\in T$. Thus $T'\models (R')^+$.

\noindent Case 2. $R'$ is of the form $e_R\la c_i$ where $1\le i\le
n$: If $c_i\in T'$, then $c_i\in S$. By this rule, $e_R\in S'$ or
$e_R\in S_e$. Thus $e_R\in T'$. Again, we have $T'\models (R')^+$.

Therefore, $T'\models (P')^{S'}$. By the minimality of $S'$,
$T'=S'$, which implies $T=S$. Thus $S$ is a minimal model of $P^S$.

So, we conclude the proof.
\end{proof}

%
%
%
%
%
%
%
\begin{appthm}{Proposition}{pro:structure}
\aspChar
\end{appthm}
\begin{proof}
\noindent $\Rightarrow$: Let $S$ be an answer set of $P$.

To prove condition 1, suppose that $b_1 \la \naf b_2$ is a rule in $P$ and $b_2 \in A_n\setminus S$. Then the rule $b_1 \la$ is in $P^S$.
This implies that $b_1 \in S$, which is in contradiction to $b_1 \in A_n\setminus S$. Therefore,  $b_1 \la \naf b_2 $ cannot be a rule in $P$ if $b_1, b_2\in A_n\setminus S$.

For condition 2, if $a \in S$, $P$ contains at least one rule with the head $a$. On the contrary, suppose that
there does not exist any $b \in A_n\setminus S$ such that  $a\la \naf b$ is in $P$. Then
for every rule of the form $a\la \naf b$ in $P$, we would have $b\in S$, which implies the reduct $P^S$ would contain no rules
whose head is $a$. Therefore, $a\not\in S$, a contradiction. Therefore, there must exist an atom $b \in A_n\setminus S$ such that  $a\la \naf b$ is in $P$.

\noindent $\Leftarrow$: Assume that $S\subseteq A_n$ satisfies the above two conditions 1 and 2.
We want show that $S$ is an answer set.

$S\models P^S$: If $R$: $a \la \naf c$ is a rule of $P$ such that $R^+ \in P^S$, then
$c\not\in S$. By condition 1, $a \in S$. This means that every rule of $P^S$ is satisfied by $S$.
Thus, $S\models P^S$.

$S$ is a minimal model of $P^S$: By condition 2, for each $a\in S$, there exists a rule $a\la \naf b$ such that $b\not\in S$.
Then the rule $a\la$ is in $P^S$, which implies that every model of $P^S$ must contain $a$. This implies that every model of $P^S$ is a superset of $S$.
Therefore, $S$ is minimal (actually the least model of $P^S$).
\end{proof}

\begin{appthm}{Proposition}{pro:no:0:n}
\aspRange
\end{appthm}
\begin{proof}
If $|S|=0$, then $S=\emptyset$. Since $P$ contains at least one
rule, we assume that $a\la \naf b$ is in $P$. By $S=\emptyset$,
$a\not\in S$. Then it would be the case that $b\in S$, which is a contradiction to $S=\emptyset$. 
Therefore, $|S|>0$.

If $|S|\neq 0$, i.e. $S\neq \emptyset$, then there exists an element $a\in S$. By the definition of answer sets,
the rule $a\la\naf b$ must be in $P$ for some $b\not\in S$. This implies that $S$ must be a proper subset of $A_n$.
\end{proof}

\subsection{Basics of  mathematical analysis}\label{sec:math_review}
In this subsection, we briefly recall some basics of mathematical analysis and notation that are used in related proofs.
\begin{enumerate} 
\item Big O notation: let $f(x)$, $g(x)$, $h(x)$ be three real functions. By
\(
f(x)=g(x)+O(h(x)),
\)
we mean that 
\(
|f(x)-g(x)| = O(h(x)).
\)
That is, there exists a positive real number $c$ and a real number $x_0$ such that for all $x>x_0$.
\[
|f(x)-g(x)| \le c|h(x)|.
\]
The same notation is also applicable to discrete functions.

\item Stirling's approximation: for all integer $n>0$
\begin{equation}\label{Stirling:1}
1 \le \frac{n!}{e^{-n}n^n\sqrt{2n\pi}} \le \frac{e}{\sqrt{2\pi}},
\end{equation}
\begin{equation}\label{Stirling:2}
n! = e^{-n}n^n\sqrt{2n\pi}(1+O(\frac{1}{n})).
\end{equation}
\item Taylor series: Let $f(x)$ be an infinitely differentiable real function on $\mathbb{R}$, $x_0 \in \mathbb{R}$ is a real number, then for all $x \in \mathbb{R}$,
\begin{equation}\label{Taylor:series}
f(x)=\sum_{i=0}^{\infty}\frac{f^{(i)}(x_0)}{i!}(x-x_0)^i.
\end{equation}
Here $f^{(i)}(x_0)$ denotes the $i$-th derivative of $f$ at $x_0$ ($i\ge 0$). In particular, $f^{(0)}(x) = f(x)$.

\item Properties of the natural exponential function:
\begin{equation}\label{exp:limit}
\lim\limits_{x\to 0}(1+x)^{\frac{1}{x}}=e.
\end{equation}
\begin{equation}\label{exp:inequality}
(1+x) \leq e^{x} \textrm{~, and the equatility holds iff $x=0$.}
\end{equation}
For all $n \in \mathbb{N}$,  
\begin{equation}\label{exp:bigO}
\left(1+\frac{1}{n}\right)^{n}=e+O\left({\frac{1}{n}}\right).
\end{equation}
\item Properties of the logarithmic function:
\begin{equation}\label{logarithmic}
\lim\limits_{x\to 0}\frac{\ln(1+x)}{x}=1
\end{equation}
\[
\text{\quad if }x>0, ~\ln(1+x)<x.               
\]

\item Concave functions: A real function $f$ is said to be \emph{concave} if, for any $x, y \in \mathbb{R}$ and for any $t$ in $[0,1]$,
\[
f(tx+(1-t)y) \ge tf(x) + (1-t)f(y).
\]
Let $f(x)$ be a continuously differentiable function. 
\begin{enumerate}
\item If $f''(x)$ is negative for all $x \in \mathbb{R}$, then $f(x)$ is a concave function.
\item For $x_0 \in \mathbb{R}$, if $f(x)$ is concave and $f'(x_0)=0$, then $f(x)$ reaches its apex at $x_0$.
\item If $f(x)$ is concave and reaches its apex at $x_0$, then $g(x)=e^{f(x)}$ is strictly monotonically increasing when $x<x_0$ and strictly monotonically decreasing when $x>x_0$.
\end{enumerate}
\item The complementary error function $\mathrm{erfc}(x)$ is defined by
\begin{equation}\label{def:erfc}
\mathrm{erfc}(x)=\frac{2}{\sqrt{\pi}}\int^{\infty}_{x}e^{-t^2}\mathrm{d}t,
\end{equation}
which has the following property:
\begin{equation}\label{prop:erfc}
\lim\limits_{x \to \infty}\mathrm{erfc}(x)=0.
\end{equation}
\end{enumerate}

\subsection{Lemmas}\label{sec:lemma}

Recall that $\phi(x)$, $x_0$ and $\sigma$ have been defined in Eq.(\ref{eq:def:phi_x}), Eq.(\ref{eq:def:x0}), and Eq.(\ref{eq:def:sigma}), respectively.
We first define three real functions as follows.
\begin{equation}\label{eq:def:psi_x}
\psi(x)=\ln(\phi(x))-\ln(\phi(x_0)).
\end{equation}

\begin{equation}\label{eq:def:xi_x}
\xi(x)=\ln(\chi(x))-\ln(\phi(x_0))=-\frac{(x-x_0)^2}{2\sigma^2}.
\end{equation}

\begin{equation}\label{eq:def:kappa_x}
\kappa(x)=1-q^{n-x}.
\end{equation}
Then 
\[
\phi(x)=\phi(x_0)e^{\psi(x)},
\]
and
\[
\chi(x)=\phi(x_0)e^{\xi(x)}.
\]
%

According to Taylor series:
\[
\ln(1+x)=0 + x - \frac{x^2}{2} + \frac{x^3}{3} - ... 
\]
We have 
\begin{equation}\label{eq:ln_q}
\ln{q}=\ln(1-\frac{c_1}{n}) = -\frac{c_1}{n}-\frac{c_1^2}{2n^2}+O\left(\frac{1}{n^3}\right)
=-\frac{c_1}{n}+O\left(\frac{1}{n^2}\right)
=O\left(\frac{1}{n}\right).
\end{equation}

\begin{lemma} \label{lem:kappa_x0}
\begin{equation}\label{eq:q_n_x0}
q^{n-x_0}=\frac{1}{\alpha}-\frac{c_1^2}{2\alpha^2n}+O\left(\frac{1}{n^2}\right).
\end{equation}
\begin{equation}\label{eq:kappa_x0}
\kappa(x_0)=\frac{\alpha-1}{\alpha}+\frac{c_1^2}{2\alpha^2n}+O\left(\frac{1}{n^2}\right).
\end{equation}
\end{lemma}
\begin{proof}

By Eq.(\ref{eq:def:x0}) and Eq.(\ref{eq:ln_q}),
\[
(n-x_0)\ln{q}=\frac{n}{\alpha}\left[\frac{-c_1}{n}-\frac{c_1^2}{2n^2}+O\left(\frac{1}{n^3}\right)\right]
=-\frac{c_1}{\alpha}-\frac{c_1^2}{2n\alpha}+O\left(\frac{1}{n^2}\right).
\]
Then
\[
q^{n-x_0}=e^{-\frac{c_1}{\alpha}}e^{-\frac{c_1^2}{2n\alpha}}e^{O\left(\frac{1}{n^2}\right)}.
\]
As $\alpha>1$ satisfies the equation $\alpha^\alpha = e^{c_1}$, we can show that 
\[
e^{-\frac{c_1}{\alpha}}=\frac{1}{\alpha}.
\] 
Note that
\[
e^x=1+x+\frac{1}{2x^2}+...
\] 
then,
\[
\begin{split}
q^{n-x_0}&=\frac{1}{\alpha}
\left[1-\frac{c_1^2}{2n\alpha}+O\left(\frac{1}{n^2}\right)\right]\left[1+O\left(\frac{1}{n^2}\right)\right]\\
&=\frac{1}{\alpha}-\frac{c_1^2}{2\alpha^2n}+O\left(\frac{1}{n^2}\right)
\end{split}.
\]
Based on the definition of $\kappa(x)$ in Eq.(\ref{eq:def:kappa_x}), 
\[
\kappa(x_0)=1-q^{n-x_0}=\frac{\alpha-1}{\alpha}+\frac{c_1^2}{2\alpha^2n}+O\left(\frac{1}{n^2}\right).
\]
\end{proof}

Remark: $\ln{q}$ and $\kappa(x_0)$ can be simplified into 
\[
\ln{q}= -\frac{c_1}{n}+O\left(\frac{1}{n^2}\right)=O\left(\frac{1}{n}\right),
\]
and
\[
\kappa(x_0)=\frac{\alpha-1}{\alpha}+O\left(\frac{1}{n}\right) = O(1).
\]

\begin{lemma} \label{lem:dpsi}
When $n$ is sufficiently large,
\begin{equation}\label{eq:dpsi_1}
\psi'(1)>0.
\end{equation}
\begin{equation}\label{eq:dpsi_n}
\psi'(n-1)<0.
\end{equation}
\begin{equation}\label{eq:dpsi_x0}
\psi'(x_0)=O\left(\frac{1}{n}\right).
\end{equation}
\end{lemma}
\begin{proof}
From the definitions of $\psi(x)$ in Eq.(\ref{eq:def:psi_x}) and $\phi(x)$ in Eq.(\ref{eq:def:phi_x}), it follows that
\begin{equation}
\label{eq:dpsi}
\begin{split}
\psi'(x)=&
\frac{1}{2}(\frac{1}{n-x}-\frac{1}{x})+\ln\frac{n-x}{x}\\
&-(2n-2x)\ln{q}+\ln(1-q^{n-x})\\
&+\frac{x}{1-q^{n-x}}\times q^{n-x}\ln{q}-\ln{r}.
\end{split}
\end{equation}
Take $x=1$ in Eq.(\ref{eq:dpsi}), we can see that the second term is $\ln(n-1)$
and all the other terms are of the order $O(1)$. So, when $n$ is sufficiently large,
\[
\psi'(1)=\ln(n-1)+O(1)>0.
\]
Take $x=n-1$ in Eq.(\ref{eq:dpsi}), the most significant term is the fifth one. By Eq.(\ref{eq:ln_q}), the fifth term can be simplified as follows.
\[
\frac{n-1}{1-q^{n-(n-1)}}\times q^{n-(n-1)}\ln{q}=-\frac{(n-1)(n-c_1)}{n}+O(1)
\]
and the other terms are of the order $O(\ln(n))$ or less. So, when $n$ is sufficient large,
\[
\psi'(n-1)=-\frac{(n-1)(n-c_1)}{n}+O(\ln{n})<0.
\]
Finally, take $x=x_0=\frac{(\alpha-1)n}{\alpha}$ in Eq.(\ref{eq:dpsi}), the first term is of the order $O\left(\frac{1}{n}\right)$. The other five terms can be simplified correspondingly into
\[
\ln\frac{n-x_0}{x_0}=-\ln(\alpha-1),
\]
\[
-(2n-2x_0)\ln{q}=\frac{2c_1}{\alpha}+O\left(\frac{1}{n}\right),
\]
\[
\ln(1-q^{n-x_0})=\ln\kappa(x)=\ln(\alpha-1)-\ln{\alpha}+O\left(\frac{1}{n}\right),
\]
\[
\frac{x_0}{1-q^{n-x_0}}\times q^{n-x_0}\ln{q}=-\frac{c_1}{\alpha}+O\left(\frac{1}{n}\right),
\]
\[
-\ln(r)=-\ln{\frac{1-d}{1-p}}=-\ln{\frac{n-c_2}{n-c_1}}=-\ln(1+\frac{c_1-c_2}{n-c_1})=O\left(\frac{1}{n}\right).
\]
Then 
\[
\psi'(x_0)=\frac{c_1}{\alpha}-\ln{\alpha}+O\left(\frac{1}{n}\right).
\]
Since $\alpha>1$ satisfies the equation $\alpha^\alpha = e^{c_1}$, we have that
\[
\frac{c_1}{\alpha}-\ln{\alpha}=0.
\]
Thus,
\[
\psi'(x_0)=O\left(\frac{1}{n}\right).
\]
\end{proof}

\begin{lemma} \label{lem:ddpsi}
If $1 \leq x \leq n-1$, then
\begin{equation}\label{eq:ddpsi_x}
\psi''(x)<2\ln{q}<0,
\end{equation}
\begin{equation}\label{eq:ddpsi_x0}
\psi''(x_0)=-\frac{1}{\sigma^2}+O\left(\frac{1}{n^2}\right)
\end{equation}
where $\sigma$ is defined in Eq.(\ref{eq:def:sigma}).
\end{lemma}
\begin{proof}
By Eq.(\ref{eq:dpsi}), 
\[
\begin{split}
\psi''(x)=&
\frac{1}{2x^2}-\frac{1}{x}
-\frac{1}{n-x}+\frac{1}{2(n-x)^2} \\
&+2\ln{q}+\frac{q^{n-x}\ln{q}}{1-q^{n-x}}\\
&+\frac{(1-q^{n-x})(q^{n-x}-xq^{n-x}\ln{q})-xq^{2(n-x)}\ln{q}}{(1-q^{n-x})^2}\times \ln{q}.
\end{split}
\]
Then it can be further simplified to
\begin{equation}
\label{eq:d2psis}
\begin{split}
\psi''(x)=&
\frac{1}{2x^2}-\frac{1}{x}
-\frac{1}{n-x}+\frac{1}{2(n-x)^2} \\
&+\frac{2\ln{q}}{\kappa(x)}+(\frac{1}{\kappa(x)}-\frac{1}{\kappa(x)^2})x\ln^2{q}.
\end{split}
\end{equation}
From $1 \leq x \leq n-1$, we have that
\[
\frac{1}{2x^2}-\frac{1}{x}
-\frac{1}{n-x}+\frac{1}{2(n-x)^2} <0.
\]
As $\kappa(x)=1-q^{n-x}$,  so $0<\kappa(x)<1$, then
\[
\left(\frac{1}{\kappa(x)}-\frac{1}{\kappa(x)^2}\right)x\ln^2{q}<0.
\]
So
\[
\psi''(x) < \frac{2\ln{q}}{\kappa(x)} < 2\ln{q} <0.
\]
Take $x=x_0=\frac{(\alpha - 1)n}{\alpha}$ in Eq.(\ref{eq:d2psis}) and split the formula into three partsas follows. Then by Eq.(\ref{eq:ln_q}) and Lemma~\ref{lem:kappa_x0},
\[
\frac{1}{2x_0^2}-\frac{1}{x_0} -\frac{1}{n-x_0}+\frac{1}{2(n-x_0)^2} = -\frac{\alpha^2}{(\alpha-1)n} + O\left(\frac{1}{n^2}\right),
\]
\[
\frac{2\ln{q}}{\kappa(x_0)}= -\frac{2\alpha c_1}{(\alpha-1)n} + O\left(\frac{1}{n^2}\right),
\]
\[
(\frac{1}{\kappa(x_0)}-\frac{1}{\kappa(x_0)^2})x_0\ln^2{q}=-\frac{c_1^2}{(\alpha-1)n} + O\left(\frac{1}{n^2}\right).
\]
Combining the three parts above together and by the definition of $\sigma$ in Eq.(\ref{eq:def:sigma}),
\[
\psi''(x_0)=-\frac{\alpha^2+2\alpha c_1 + c_1^2}{(\alpha-1)n} +O\left(\frac{1}{n^2}\right) = -\frac{1}{\sigma^2}+O\left(\frac{1}{n^2}\right).
\]
\end{proof}

\begin{lemma} \label{lem:dipsi}
For all $i > 2$, the $i$-th derivative of $\psi(x)$ at $x_0$ satisfies
\begin{equation}\label{eq:dipsi_x0}
\psi^{(i)}(x_0) = O\left(\frac{1}{n^{i-1}}\right).
\end{equation}
\end{lemma}
\begin{proof}
By the definition of $\kappa(x)$ in Eq.(\ref{eq:def:kappa_x}), for $i>2$, we have that
\[
\kappa^{(i)}(x)=(-1)^{(i+1)}q^{n-x}(\ln{q})^i.
\]
Take $x=x_0$, by Eq.(\ref{eq:ln_q}) and Lemma~\ref{lem:kappa_x0}, the formula above can be simplified to
\begin{equation}\label{eq:di:kappa}
\kappa^{(i)}(x_0)=O\left(\frac{1}{n^i}\right).
\end{equation}
Define 
\[
\psi_1(x)=\frac{1}{2x^2}-\frac{1}{x} -\frac{1}{n-x}+\frac{1}{2(n-x)^2},
\]
\[
\psi_2(x)=\frac{2\ln{q}}{\kappa(x)},
\]
\[
\psi_3(x)=\left(\frac{1}{\kappa(x)}-\frac{1}{\kappa(x)^2}\right)x\ln^2{q}.
\]
By Eq.(\ref{eq:d2psis}),
\[
\psi''(x)=\psi_1(x)+\psi_2(x)+\psi_3(x).
\]
Thus
\[
\psi^{(i)}(x)=\psi_1^{(i-2)}(x)+\psi_2^{(i-2)}(x)+\psi_3^{(i-2)}(x).
\]
As $x_0=O(n)$, we have that 
\begin{equation}\label{eq:p:dipsi_1}
\psi_1^{(i-2)}(x_0)= O\left(\frac{1}{n^{i-1}}\right).
\end{equation}
Then
\[
\psi_2(x)=2\ln{q}\left(\kappa(x)\right)^{-1},
\]
\[
\psi_2'(x)=2\ln{q}\left(-\kappa(x)^{-2}\kappa'(x)\right),
\]
\[
\psi_2''(x)=2\ln{q}\left(2\kappa(x)^{-3}\kappa'(x)^2-\kappa(x)^{-2}\kappa''(x)\right),
\]
\[
\psi_2'''(x)=2\ln{q}\left(-6\kappa(x)^{-4}\kappa'(x)^3+6\kappa(x)^{-3}\kappa'(x)\kappa''(x).
-\kappa(x)^{-2}\kappa'''(x)\right).
\]
In general, for $i>2$, it holds that
\[
\psi_2^{(i-2)}(x)=2\ln{q}\sum_j\Lambda_{i,j}(x)
\]
where
\[
\Lambda_{i,j}(x)=c_{i,j}\kappa(x)^{-j}\prod_s(\kappa^{(i_s)}(x))^{t_s},
\]
$c_{i,j}$ is a constant determined by $i$ and $j$, and
\[
\sum_s{i_s \times t_s}=i-2.
\]
Then by Eq.(\ref{eq:di:kappa}), we know that 
\[
\Lambda_{i,j}(x_0) = O\left(\frac{1}{n^{i-2}}\right).
\]
Since $\ln{q}=O\left(\frac{1}{n}\right)$, 
\[
\psi_2^{(i-2)}(x_0)=O\left(\frac{1}{n^{i-1}}\right).
\]
Similarly, we can show that
\[
\psi_3^{(i-2)}(x_0)=O\left(\frac{1}{n^{i-1}}\right).
\]
Therefore,
\[
\psi^{(i)}(x_0)=\psi_1^{(i-2)}(x_0)+\psi_2^{(i-2)}(x_0)+\psi_3^{(i-2)}(x_0)=O\left(\frac{1}{n^{i-1}}\right).
\]

\end{proof}

\begin{lemma} \label{lem:phi:x0}
\begin{equation}\label{eq:phi:x0}
\phi(x_0)=\frac{\alpha e^{\frac{c_1-c_2}{\alpha}}}{\sqrt{2\pi(\alpha-1)n}}+O(n^{-\frac{3}{2}}).
\end{equation}
\end{lemma}

\begin{proof}
By the definition of $\phi(x)$ and $x_0$ in  Eq.(\ref{eq:def:phi_x}) and Eq.(\ref{eq:def:x0}), 
\[
\begin{split}
\phi(x_0)&=\sqrt{\frac{n}{2\pi x_0(n-x_0)}}
\left(\frac{n(1-q^{n-x_0})}{x_0}\right)^{x_0}\left(\frac{nrq^{n-x_0}}{n-x_0}\right)^{n-x_0}\\
&=\sqrt{\frac{\alpha^2}{2\pi(\alpha-1)n}} 
\left(\frac{\kappa(x_0)}{\frac{\alpha-1}{\alpha}}\right)^{\frac{\alpha-1}{\alpha}n}
\left(\frac{rq^{n-x_0}}{\frac{1}{\alpha}}\right)^{\frac{n}{\alpha}}.\\
\end{split}
\]
By Lemma~\ref{lem:kappa_x0},
\[
\phi(x_0)=\sqrt{\frac{\alpha^2}{2\pi(\alpha-1)n}} 
\left(1+\frac{c_1^2}{2\alpha(\alpha-1)n}+O\left(\frac{1}{n^2}\right)\right)^{\frac{\alpha-1}{\alpha}n}
\left(\left(1-\frac{c_2-c_1}{n-c_1}\right)\left(1-\frac{c_1^2}{2\alpha n}+O\left(\frac{1}{n^2}\right)\right)\right)^{\frac{n}{\alpha}}.
\]
Then Eq.(\ref{exp:bigO}), the above equation can be further simplified as follows:  
\[
\begin{split}
\phi(x_0)&=\sqrt{\frac{\alpha^2}{2\pi(\alpha-1)n}} 
\left(e^{\frac{c_1^2}{2\alpha^2}}+O\left(\frac{1}{n}\right)\right)
\left(e^{-\frac{c_1^2}{2\alpha^2}+\frac{c_1-c_2}{\alpha}}+O\left(\frac{1}{n}\right)\right)\\
&=\frac{\alpha e^{\frac{c_1-c_2}{\alpha}}}{\sqrt{2\pi(\alpha-1)n}}+O(n^{-\frac{3}{2}}).
\end{split}
\]
\end{proof}

\begin{lemma} \label{lem:phi:small_at_edge}
Let $\Delta=c_0\sqrt{n\ln{n}}$ as defined in Eq.(\ref{eq:def:Delta}), where $c_0$ is defined in Eq.(\ref{eq:def:c0}). Then
\[
\lim\limits_{n \to \infty}\left(\int_{1}^{x_0- \Delta}\phi(x) \mathrm{d}x +
\int_{x_0+\Delta}^{n-1}\phi(x) \mathrm{d}x\right) = 0,
\]
and
\[
\lim\limits_{n \to \infty}\left(\sum\limits_{k=1}^{\lfloor{x_0-\Delta}\rfloor}\phi(k) +
\sum\limits_{k=\lfloor{x_0+\Delta}\rfloor}^{n-1}\phi(k) \right)=0.
\]
\end{lemma}
\begin{proof}
By the definition of $\psi(x)$, we have
\[
\phi(x)=\phi(x_0)e^{\psi(x)}.
\]
By Lemma~\ref{lem:dpsi} and Lemma~\ref{lem:ddpsi}, for all $x \in [1,n-1]$,
\[
\psi''(x)<2\ln{q}=-\frac{2c_1}{n}+O(\frac{1}{n^2}).
\]
Then the Taylor series for $\psi$ at $x \in [1,n-1]$ is
\[
\psi(x)\leq \psi(x_0)+(x-x_0)\psi'(x_0)+\frac{1}{2}(x-x_0)^2\max(\psi''(x)).
\]
As $x<n$, so
\[
\psi(x)\leq -\frac{c_1}{n}(x-x_0)^2+O(1).
\]
We note that the function $-\frac{c_1}{n}(x-x_0)^2$ is an upper bound for $\psi(x)$, which is strictly increasing when $x<x_0$ and strictly decreasing when $x>x_0$. Thus, 
\[
\begin{split}
\int^{x_0-\Delta}_1 \phi(x) \mathrm{d}x+ \int_{x_0+\Delta}^{n-1} \phi(x) \mathrm{d}x &\leq
\int^{x_0-\Delta}_1 \phi(x_0-\Delta) \mathrm{d}x+ \int_{x_0+\Delta}^{n-1} \phi(x_0+\Delta) \mathrm{d}x\\
&\leq\phi(x_0)ne^{-\ln{n}+O(1)}
=O(\phi(x_0)).
\end{split}
\]
By Lemma~\ref{lem:phi:x0}, 
\[
\phi(x_0)=\frac{\alpha e^{\frac{c_1-c_2}{\alpha}}}{\sqrt{2\pi(\alpha-1)n}}+O\left(n^{-\frac{3}{2}}\right)
=O\left(\frac{1}{\sqrt{n}}\right).
\]
So,
\[
\lim\limits_{n \to \infty}\left(\int_{1}^{x_0- \Delta}\phi(x) \mathrm{d}x +
\int_{x_0+\Delta}^{n-1}\phi(x) \mathrm{d}x\right) \leq 0.
\]
From $\phi(x)\geq 0$, it follows that 
\[
\lim\limits_{n \to \infty}\left(\int_{1}^{x_0- \Delta}\phi(x) \mathrm{d}x +
\int_{x_0+\Delta}^{n-1}\phi(x) \mathrm{d}x\right) = 0.
\]
Thus,
\[
\lim\limits_{n \to \infty}\left(\sum\limits_{k=1}^{\lfloor{x_0-\Delta}\rfloor}\phi(k) +
\sum\limits_{k=\lfloor{x_0+\Delta}\rfloor}^{n-1}\phi(k) \right)=0.
\]
\end{proof}

The next lemma is a basic property of integral. We present it here for reader's reference.
\begin{lemma}\label{lem:sum:int}
Let the function $\phi$ be defined as in Eq.(\ref{eq:def:phi_x}). Then
\begin{equation} 
\lim\limits_{n \to \infty}\sum\limits_{k=1}^{n-1} \phi(k) = \lim\limits_{n \to \infty}\int^{n}_{1} \phi(x) \mathrm{d}x.
\end{equation}
\end{lemma}

\begin{proof}
By Lemma~\ref{lem:ddpsi}, $\psi''(x)<0$ when $x \in [1,n-1]$. We know that $\psi(x)$ is a concave function in the range. Also, by Lemma~\ref{lem:dpsi}, $\psi'(1)>0$ and $\psi'(n-1) < 0$, which mean there exists a unique $\hat{x} \in (1, n-1)$ such that $\psi(x)$  reaches its apex at $\hat{x}$. As $\phi(x)=\phi(x_0)e^{\psi(x)}$ and it is a concave function, 
$\phi(x)$ is strictly increasing for $x\in(1,\hat{x})$ and strictly decreasing for $x\in(\hat{x}, n-1)$. 

\begin{figure}[h]
\centering
\includegraphics[height=40mm]{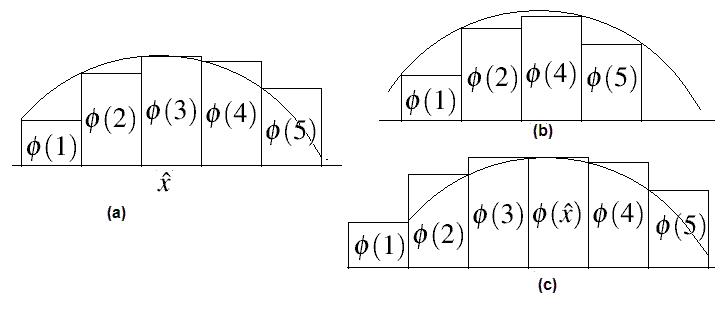}
\caption{The integral and its approximation}
\label{fig:sum:int}
\end{figure}

To compare the difference between the integral and the sum of the discrete values, we use  Figure~\ref{fig:sum:int} as an example. The curve reaches its maximum at $\hat{x}$ which is larger than 3 and smaller than 4. Clearly, from  Figure~\ref{fig:sum:int}(a), it is difficult to compare the integral and the sum of the discrete values. However, if we remove the tallest bar, which is $\phi(3)$ and shift all the bars right of it leftward one step, then clearly (as shown in Figure~\ref{fig:sum:int}(b)) the sum of the discrete values is smaller than the integral of the curve. If we insert the bar of $\phi(\hat{x})$ at the left of the bar of the smallest number which is larger than $\hat{x}$, (in this example it is 4), and shift all the bars left of it leftward one step (as shown in Figure~\ref{fig:sum:int}(c)), then the total of the discrete values is larger than the integral of the curve. Therefore we have:

\[
\sum\limits_{k=1}^{n-1} \phi(k) +\phi(\hat{x})> \int^{n}_{1} \phi(x)\mathrm{d}x> \sum\limits_{k=1}^{n-1} \phi(k) - \phi(\hat{x}).
\]

By Lemma~\ref{lem:phi:x0}, we know that
\[
\phi(x_0)=\frac{\alpha e^{\frac{c_1-c_2}{\alpha}}}{\sqrt{2\pi(\alpha-1)n}}+O\left(n^{-\frac{3}{2}}\right)
=O\left(\frac{1}{\sqrt{n}}\right).
\]
Also, from the proof of Lemma~\ref{lem:phi:small_at_edge}, we can see that, for $x\in [1,n-1]$, $\phi(x)=\phi(x_0)e^{\psi(x)}$ and $\psi(x) \leq -c_1(x-x_0)^2/n+ O(1)$. So, $\phi(x)=O(\phi(x_0))$, which implies $\phi(\hat{x})=O(n^{-\frac{1}{2}})$. That means $\lim\limits_{n \to \infty}\phi(\hat{x})=0$. Therefore, 
\[
\lim\limits_{n \to \infty}\sum\limits_{k=1}^{n-1} \phi(k) 
= \lim\limits_{n \to \infty}\int^{n}_{1} \phi(x)\mathrm{d}x.
\]
\end{proof}

%
\begin{lemma}\label{lem:cond1}
\begin{equation}
\label{eq:cond1}
\lim\limits_{n \to \infty}\left(\int^{x_0-\Delta}_{-\infty} \chi(x)\mathrm{d}x + \int_{x_0+\Delta}^{\infty} \chi(x) \mathrm{d}x\right)= 0.
\end{equation}
\end{lemma}
%
\begin{proof}
Note that 
\[
\begin{split}
\int^{x_0-\Delta}_{-\infty} \chi(x)\mathrm{d}x + \int_{x_0+\Delta}^{\infty} \chi(x) \mathrm{d}x
&=2 \int_{x_0+\Delta}^{\infty} \chi(x) \mathrm{d}x\\
&=2\phi(x_0)\int_{\Delta}^{\infty} e^{-\frac{x^2}{2\sigma^2}} \mathrm{d}x.\\
\end{split}
\]
Let $x=\sqrt{2}\sigma t$, then
\[
\int_{\Delta}^{\infty} e^{-\frac{x^2}{2\sigma^2}} \mathrm{d}x=
\sqrt{2}\sigma\int_{\frac{\Delta}{\sqrt{2}\sigma}}^{\infty} e^{-t^2} \mathrm{d}t=\frac{\sqrt{2\pi}\sigma}{2}\mathrm{erfc}(\frac{\Delta}{\sqrt{2}\sigma}),
\]
where $\mathrm{erfc}$ is the complementary error function. Then
\[
\lim\limits_{n \to \infty}\left(\int^{x_0-\Delta}_{-\infty} \chi(x)\mathrm{d}x + \int_{x_0+\Delta}^{\infty} \chi(x) \mathrm{d}x\right)=\lim\limits_{n \to \infty}\sqrt{2\pi}\sigma\phi(x_0)\mathrm{erfc}(\frac{\Delta}{\sqrt{2}\sigma}).
\]
By Eq.(\ref{eq:def:sigma}) and Lemma~\ref{lem:phi:x0}, we know that 
\[
\sigma\phi(x_0)=O(1).
\]
And then by Eq.(\ref{eq:def:Delta}), we have
\[
\frac{\Delta}{\sqrt{2}\sigma}=O\left(\sqrt{\ln{n}}\right) \to \infty.
\]
From Eq.\ref{prop:erfc} (the property of complementary error function), it follows that
\[
\lim\limits_{z \to \infty}\mathrm{erfc}(z)=0.
\]
Thus,
\[
\lim\limits_{n \to \infty}\left(\int^{x_0-\Delta}_{-\infty} \chi(x)\mathrm{d}x + \int_{x_0+\Delta}^{\infty} \chi(x) \mathrm{d}x\right)=0.
\]
\end{proof}

%
\begin{lemma}\label{lem:cond2}
\begin{equation}
\label{eq:cond2}
\lim\limits_{n \to \infty}\int^{x_0+\Delta}_{x_0-\Delta} |\phi(x)- \chi(x)|\mathrm{d}x= 0.
\end{equation}
\end{lemma}
%
\begin{proof}
From the definitions of $\phi(x)$ and $\chi(x)$ in Eq.(\ref{eq:def:phi_x}) and Eq.(\ref{eq:def:chi_x}),
\[
\begin{split}
\int^{x_0+\Delta}_{x_0-\Delta} |\phi(x) - \chi(x) |\mathrm{d}x&=\phi(x_0)\int^{x_0+\Delta}_{x_0-\Delta}|e^{\psi(x)}-e^{\xi(x)}|\mathrm{d}x\\
&=\phi(x_0)\int^{x_0+\Delta}_{x_0-\Delta} e^{\xi(x)} |e^{\psi(x)-\xi(x)}-1|\mathrm{d}x,
\end{split}
\]
where 
\[
\xi(x)=-\frac{(x-x_0)^2}{2\sigma^2}.
\]
Note that $e^{\xi(x)}\leq 1$ and when $|\delta|$ is small enough,
\[
|e^{\delta}-1|\leq 2|\delta|.
\]
If we can show that $\psi(x)-\xi(x) \to 0$ when $x \in [x_0-\Delta, x_0+\Delta]$, then
\begin{equation}
\label{eq:cond4}
\int^{x_0+\Delta}_{x_0-\Delta} |\phi(x) - \chi(x) |\mathrm{d}x
\leq 2\phi(x_0)\int^{x_0+\Delta}_{x_0-\Delta} |\psi(x)-\xi(x)|\mathrm{d}x.
\end{equation}
From the definition of $\xi(x)$, it follows that
\[
\xi(x_0)=\xi'(x_0)=0,
\]
\[
\xi''(x_0)=-\frac{1}{\sigma^2},
\]
and
\[
\xi^{(i)}(x_0)=0, {\rm for~ } i>2.
\]
From the definition of $\psi(x)$ in Eq.(\ref{eq:def:psi_x}), 
\[
\psi(x_0)=0.
\]
By Lemma~\ref{lem:dpsi},  Lemma~\ref{lem:ddpsi} and Lemma~\ref{lem:dipsi},
\[
\psi'(x_0)=O\left(\frac{1}{n}\right),
\]
\[
\psi''(x_0)=-\frac{1}{\sigma^2}+O\left(\frac{1}{n^2}\right),
\]
\[
\psi^{(i)}(x_0)=O\left(n^{-(i-1)}\right), {\rm for~ } i>2.
\]
Based on the Taylor series for the function $\psi(x)-\xi(x)$,
\[
|\psi(x)-\xi(x)| \leq \sum\limits_{i=0}^{\infty} \left|\frac{\psi^{(i)}(x_0)-\xi^{(i)}(x_0)}{i!}(x-x_0)^i\right|.
\]
As $x \in [x_0-\Delta, x_0+\Delta]$,
\[
|x-x_0| \leq O\left(\sqrt{n\ln{n}}\right).
\]
Thus, 
\[
|\psi(x)-\xi(x)| \leq O\left(\sqrt{\frac{\ln {n}}{n}}\right) \to 0,
\]
which indicates Eq(\ref{eq:cond4}) holds.  
By Eq(\ref{eq:cond4}), we have
\[
\begin{split}
\int^{x_0+\Delta}_{x_0-\Delta} |\phi(x) - \chi(x) | \mathrm{d}x
&\leq O(2\phi(x_0))\int^{x_0+\Delta}_{x_0-\Delta} |\psi(x)-\xi(x)|\mathrm{d}x\\
&\leq O\left(\frac{1}{\sqrt{n}}\right) \times O\left(\sqrt{n \ln n}\right) \times O\left(\sqrt{\frac{\ln {n}}{n}}\right) \\
& =O\left( \frac{\ln n}{\sqrt{n}}\right) \to 0.
\end{split}
\]
\end{proof}

%% file: linear_asp.bbl
\begin{thebibliography}{}

\bibitem[\protect\citeauthoryear{Achlioptas, Kirousis, Kranakis, Krizanc,
  Molloy, and Stamatiou}{Achlioptas et~al\mbox{.}}{1997}]{AchlioptasKKKMS97}
{\sc Achlioptas, D.}, {\sc Kirousis, L.}, {\sc Kranakis, E.}, {\sc Krizanc,
  D.}, {\sc Molloy, M.}, {\sc and} {\sc Stamatiou, Y.} 1997.
\newblock Random constraint satisfaction: A more accurate picture.
\newblock In {\em Proceedings of the 3rd International Conference on Principles
  and Practice of Constraint Programming (CP-97)}. 107--120.

\bibitem[\protect\citeauthoryear{Achlioptas, Naor, and Peres}{Achlioptas
  et~al\mbox{.}}{2005}]{AchlioptasNP05}
{\sc Achlioptas, D.}, {\sc Naor, A.}, {\sc and} {\sc Peres, Y.} 2005.
\newblock Rigorous location of phase transitions in hard optimization problems.
\newblock {\em Nature\/}~{\em 435,\/}~7043, 759–764.

\bibitem[\protect\citeauthoryear{Baral, Gelfond, and Rushton}{Baral et~al\mbox{.}}{2009}]{P-Log}
{\sc Baral, C.}, {\sc Gelfond, M.} {\sc and} {\sc Rushton, N.} 2009.
\newblock Probabilistic reasoning with answer sets.
\newblock {\em Theory and Practice of Logic Programming\/}~{\em 9,\/}~1, 57--144.

\bibitem[\protect\citeauthoryear{Blair, Dushin, Jakel, Rivera, and Sezgin}{Blair
  et~al\mbox{.}}{1999}]{BlairDJRS99}
{\sc Blair, H.}, {\sc Dushin, F.}, {\sc Jakel, D.}, {\sc Rivera, D.}, {\sc and} {\sc Sezgin, M.} 1999.
\newblock Continuous models of computation for logic programs:
importing continuous mathematics into logic programming's algorithmic
foundations.
\newblock In {\em The Logic Programming Paradigm}. 231--255.

\bibitem[\protect\citeauthoryear{Brass and Dix}{Brass and Dix}{1999}]{BrassD99}
{\sc Brass, S.} {\sc and} {\sc Dix, J.} 1999.
\newblock Semantics of disjunctive logic
                  programs based on partial evaluation.
\newblock {\em Journal of Logic Programming\/}~{\em 38,\/}~3, 167--312.

\bibitem[\protect\citeauthoryear{Cheeseman, Kanefsky, and Taylor}{Cheeseman
  et~al\mbox{.}}{1991}]{CheesemanKT91}
{\sc Cheeseman, P.}, {\sc Kanefsky, B.}, {\sc and} {\sc Taylor, W.~M.} 1991.
\newblock Where the really hard problems are.
\newblock In {\em Proceedings of the 12th International Joint Conference on
  Artificial Intelligence (IJCAI-91)}. 331--340.

\bibitem[\protect\citeauthoryear{Gebser, Kaufmann, and Schaub}{Gebser
  et~al\mbox{.}}{2009}]{GebserKS09}
{\sc Gebser, M.}, {\sc Kaufmann, B.}, {\sc and} {\sc Schaub, T.} 2009.
\newblock The conflict-driven answer set solver clasp: Progress report.
\newblock In {\em Proceedings of the 10th International Conference on Logic
  Programming and Nonmonotonic Reasoning (LPNMR-09)}. 509--514.
  
\bibitem[\protect\citeauthoryear{Gelfond and Lifschitz}{Gelfond and
  Lifschitz}{1990}]{gellif90a}
{\sc Gelfond, M.} {\sc and} {\sc Lifschitz, V.} 1990.
\newblock The stable model semantics for logic programming.
\newblock In {\em Proceedings of the 5th International Conference on Logic
  Programming (ICLP-88)}. 1070--1080.

\bibitem[\protect\citeauthoryear{Gelfond and Lifschitz}{Gelfond and
  Lifschitz}{1988}]{gellif88}
{\sc Gelfond, M.} {\sc and} {\sc Lifschitz, V.} 1988.
\newblock Logic programs with classical negation.
\newblock In {\em Proceedings of the 7th International Conference on Logic
  Programming (ICLP-90)}. 579--597.

\bibitem[\protect\citeauthoryear{Gent and Walsh}{Gent and
  Walsh}{1994}]{GentW94}
{\sc Gent, I.} {\sc and} {\sc Walsh, T.} 1994.
\newblock The sat phase transition.
\newblock In {\em Proceedings of the Eleventh European Conference on Artificial
  Intelligence (ECAI-94)}. 105--109.

\bibitem[\protect\citeauthoryear{Huang, Jia, C., and You}{Huang
  et~al\mbox{.}}{2002}]{HuangJLY02}
{\sc Huang, G.}, {\sc Jia, X.}, {\sc C.}, {\sc and} {\sc You, J.} 2002.
\newblock Two-literal logic programs and satisfiability representation of
  stable models: A comparison.
\newblock In {\em Proceedings 15th Canadian Conference on Artificial
  Intelligence}. 119--131.

\bibitem[\protect\citeauthoryear{Huberman and Hogg}{Huberman and
  Hogg}{1987}]{HubermanH87}
{\sc Huberman, B.} {\sc and} {\sc Hogg, T.} 1987.
\newblock Phase transitions in artificial intelligence systems.
\newblock {\em Artificial Intelligence\/}~{\em 33,\/}~2, 155--171.

\bibitem[\protect\citeauthoryear{Janhunen}{Janhunen}{2006}]{Janhunen06}
{\sc Janhunen, T.} 2006.
\newblock Some (in)translatability results for normal logic programs and propositional theories.
\newblock {\em Journal of Applied Non-Classical Logics \/}~{\em ,\/}~1-2,
35--86.

\bibitem[\protect\citeauthoryear{Leone, Pfeifer, Faber, Eiter, Gottlob, Perri,
  and Scarcello}{Leone et~al\mbox{.}}{2006}]{LeonePFEGPS06}
{\sc Leone, N.}, {\sc Pfeifer, G.}, {\sc Faber, W.}, {\sc Eiter, T.}, {\sc
  Gottlob, G.}, {\sc Perri, S.}, {\sc and} {\sc Scarcello, F.} 2006.
\newblock The dlv system for knowledge representation and reasoning.
\newblock {\em ACM Transactions on Computational Logic\/}~{\em 7,\/}~3,
  499--562.
  
\bibitem[\protect\citeauthoryear{Lonc and Truszczynski}{Lonc and Truszczynski}{2002}]{LoncT04}
{\sc Lonc, Z.} {\sc and} {\sc Truszczynski, M.} 2004.
\newblock Computing stable models: worst-Case performance estimates.
\newblock {\em Theory and Practice of Logic Programming\/}~{\em 4,\/}~1-2,
  193--231.

\bibitem[\protect\citeauthoryear{Marek and Truszczynski}{Marek and Truszczynski}{1991}]{MarekT91}
{\sc Marek, V.} {\sc and} {\sc Truszczynski, M.} 1991.
\newblock Autoepistemic logic.
\newblock {\em Journal of the Association for Computing
Machinery\/}~{\em 38,\/}~3,
588--619.

\bibitem[\protect\citeauthoryear{Marek and Truszczynski}{Marek and Truszczynski}{1993}]{MarekT93}
{\sc Marek, V.} {\sc and} {\sc Truszczynski, M.} 1993.
\newblock Nonmonotonic Logic: Context-Dependent Reasonong.
\newblock {\em Springer}, 1993.

\bibitem[\protect\citeauthoryear{Mitchell, Selman, and Levesque}{Mitchell
  et~al\mbox{.}}{1992}]{MitchellSL92}
{\sc Mitchell, D.}, {\sc Selman, B.}, {\sc and} {\sc Levesque, H.} 1992.
\newblock Hard and easy distributions of sat problems.
\newblock In {\em Proceedings of the 10th National Conference on Artificial
  Intelligence (AAAI-92)}. 459--465.

\bibitem[\protect\citeauthoryear{Monasson, Zecchina, Kirkpatrick, Selman, and
  Troyansky}{Monasson et~al\mbox{.}}{1999}]{MonassonZKST99}
{\sc Monasson, R.}, {\sc Zecchina, R.}, {\sc Kirkpatrick, S.}, {\sc Selman,
  B.}, {\sc and} {\sc Troyansky, L.} 1999.
\newblock 2+p-sat: Relation of typical-case complexity to the nature of the
  phase transition.
\newblock {\em Random Structures and Algorithms\/}~{\em 15,\/}~3-4, 414--435.

\bibitem[\protect\citeauthoryear{Namasivayam}{Namasivayam}{2009}]{Namasivayam09}
{\sc Namasivayam, G.} 2009.
\newblock Study of random logic programs.
\newblock In {\em Proceedings of the 25th International Conference on Logic
  Programming (ICLP-09)}. 555--556.

\bibitem[\protect\citeauthoryear{Namasivayam and Truszczynski}{Namasivayam and
  Truszczynski}{2009}]{NamasivayamT09}
{\sc Namasivayam, G.} {\sc and} {\sc Truszczynski, M.} 2009.
\newblock Simple random logic programs.
\newblock In {\em Proceedings of the 10th International Conference on Logic
  Programming and Nonmonotonic Reasoning (LPNMR-09)}. 223--235.

\bibitem[\protect\citeauthoryear{Schlipf, Truszczynski, and Wong}{Schlipf
  et~al\mbox{.}}{2005}]{SchlipfTW2005}
{\sc Schlipf, J.}, {\sc Truszczynski, M.}, {\sc and} {\sc Wong, D.} 2005.
\newblock On the distribution of programs with stable models.
\newblock In {\em 05171 Abstracts Collection - Nonmonotonic Reasoning, Answer
  Set Prorgamming and Constraints}.
  
\bibitem[\protect\citeauthoryear{Staab and Studer}{Staab and Studer}{2004}]{StaabS04}
{\sc Staab, S.} {\sc and} {\sc Studer, R.} 2004.
\newblock Handbook on Ontologies.
\newblock {\em Springer-Verlag}, 2004.

\bibitem[\protect\citeauthoryear{Syrj{\"a}nen and Niemel{\"a}}{Syrj{\"a}nen and
  Niemel{\"a}}{2001}]{SyrjanenN01}
{\sc Syrj{\"a}nen, T.} {\sc and} {\sc Niemel{\"a}, I.} 2001.
\newblock The smodels system.
\newblock In {\em Proceedings of the 6th International ConferenceLogic Logic
  Programming and Nonmonotonic Reasoning (LPNMR-01)}. 434--438.

\bibitem[\protect\citeauthoryear{Wang and Zhou}{Wang and Zhou}{2005}]{WangZ05}
{\sc Wang, K.} {\sc and} {\sc Zhou, L.} 2005.
\newblock Comparisons and computation of well-founded semantics
              for disjunctive logic programs.
\newblock {\em ACM Transactions on Computational Logic\/}~{\em 6,\/}~2,
295--327.  


\bibitem[\protect\citeauthoryear{Zhao and Lin}{Zhao and Lin}{2003}]{ZhaoL03}
{\sc Zhao, Y.} {\sc and} {\sc Lin, F.} 2003.
\newblock Answer set programming phase transition: A study on randomly
  generated programs.
\newblock In {\em Proceedings of the 19th International Conference on Logic
  Programming (ICLP-03)}. 239--253.

\end{thebibliography}
